\documentclass[journal]{IEEEtran}
\IEEEoverridecommandlockouts

\usepackage{algorithm}
\usepackage{algpseudocode}
\usepackage{amssymb}
\usepackage{amsfonts}
\usepackage{bm}
\usepackage{amsmath,amssymb,amsfonts}
\usepackage{amsthm}
\usepackage{algorithm}
\usepackage{algpseudocode}
\usepackage{mathtools} 
\usepackage[colorlinks=True,
            linkcolor=black,
            anchorcolor=green,
            citecolor=black
            ]{hyperref} 
\usepackage[noabbrev,capitalise,nameinlink]{cleveref}
\usepackage{float}
\usepackage{graphicx}
\usepackage{xcolor}
\usepackage{epstopdf}
\usepackage{multirow}
\usepackage{caption,subcaption}
 \usepackage{enumitem}
 \usepackage{bm}
\usepackage{enumitem}
\usepackage[normalem]{ulem}
\usepackage{cases}
\usepackage{color}
\usepackage{geometry}
\usepackage{array}
\usepackage{etoolbox}
 
\makeatletter
\patchcmd{\@algocf@start}
  {-1.5em}
  {0pt}
  {}{}
\makeatother

\usepackage{xpatch}

\makeatletter
\xpatchcmd\SetKwInOut
  {\hangafter=1\parbox[t]}
  {\hangafter=1\justify\parbox[t]}
\makeatother
\usepackage{ragged2e}

\usepackage{cite}
\usepackage{amsmath,amssymb,amsfonts}
\usepackage{algorithm}
\usepackage{algpseudocode}
\usepackage{graphicx}
\usepackage{caption}
\usepackage{subcaption}
\usepackage[compact]{titlesec}  
\titlespacing{\section}{2pt}{2pt}{2pt}
\usepackage{textcomp}
\usepackage{xcolor}
\usepackage{comment}
\usepackage{hyperref}
\def\BibTeX{{\rm B\kern-.05em{\sc i\kern-.025em b}\kern-.08em
    T\kern-.1667em\lower.7ex\hbox{E}\kern-.125emX}}
\usepackage{bm}

\newtheorem{lemma}{Lemma}[section]

\newtheorem{remark}{Remark}[section]
\newtheorem{assumption}{Assumption}[section]

\graphicspath{ {./images/} }

\ifCLASSINFOpdf
\else
\fi

\ifCLASSOPTIONcompsoc
\else
  \usepackage{cite}
\fi

\begin{document}

\title{Latency-aware Multimodal Federated Learning over UAV Networks
}

\author{ 
{Shaba Shaon and Dinh C. Nguyen}
\IEEEcompsocitemizethanks{
\IEEEcompsocthanksitem *Part of this work has been accepted at the IEEE Conference on Standards for Communications and
Networking (CSCN), Serbia, 2024 \cite{shaon2024wireless}. Shaba Shaon and Dinh C. Nguyen are with ECE Department,  University of Alabama in Huntsville, Huntsville, AL 35899, USA. emails: ss0670@uah.edu, dinh.nguyen@uah.edu.
}}

\maketitle

\begin{abstract}
This paper investigates federated multimodal learning (FML) assisted by unmanned aerial vehicles (UAVs) with a focus on minimizing system latency and providing convergence analysis. In this framework, UAVs are distributed throughout the network to collect data, participate in model training, and collaborate with a base station (BS) to build a global model. By utilizing multimodal sensing, the UAVs overcome the limitations of unimodal systems, enhancing model accuracy, generalization, and offering a more comprehensive understanding of the environment. The primary objective is to optimize FML system latency in UAV networks by jointly addressing UAV sensing scheduling, power control, trajectory planning, resource allocation, and BS resource management. To address the computational complexity of our latency minimization problem, we propose an efficient iterative optimization algorithm combining block coordinate descent and successive convex approximation techniques, which provides high-quality approximate solutions. We also present a theoretical convergence analysis for the UAV-assisted FML framework under a non-convex loss function. Numerical experiments demonstrate that our FML framework outperforms existing approaches in terms of system latency and model training performance under different data settings.
\end{abstract}
\vspace{-3pt}

\section{Introduction}
Unmanned aerial vehicles (UAVs), commonly known as drones, have been revolutionizing next-generation wireless networks with their versatile capabilities including line-of-sight (LoS) connections, 3D mobility, and flexibility. UAVs can act as flying base stations (BSs) for delivering communication, computation, and caching services to overcome traditional infrastructure limitations, while can also serve as flying users for tasks like remote sensing, delivery services, target tracking, and virtual reality support. Recently, UAVs have been integrated with machine learning (ML) for intelligent services, such as classifying aerial images from UAV cameras. To ensure data privacy during ML model training in UAV networks, federated learning (FL) has been recently employed, allowing UAVs to train a model and share only the model updates with a cloud server, without exchanging data. \cite{karmakar2023novel}. \textcolor{black}{More specifically, in FL over UAV networks, each UAV trains a local model on its own dataset, which may consist of images or sensor data collected during flights. The UAVs then send only the updated local model parameters to a central server, which aggregates these updates to refine the global model. This decentralized approach ensures that sensitive data remains on the UAV, enhancing privacy and security \cite{lim2021towards, zhang2020federated, liu2023federated}.}

Given the diversity of data sources (e.g., visual, auditory and textual data) in real-life intelligent UAV applications, federated multimodal learning (FML) \cite{chen2022towards} has recently introduced as a promising solution to collaborate different UAVs across different modality clusters. This approach leverages UAVs' diverse sensing capabilities to provide complementary data for improved model accuracy and generalization. By collaboratively processing different data types, UAVs can better understand and respond to complex scenarios, addressing single-modal system limitations. \textcolor{black}{In FML over UAV networks, each UAV processes and trains on data of specific modality, then contributes model updates tailored to that modality's characteristics to the central server. This approach allows for a richer aggregation process that leverages the strengths of each data type, distinguishing FML-UAV from FL-UAV by providing a more comprehensive representation of complex environments \cite{dong2023federated, gong2024multi}.}
\subsection{Related Works}
Several studies have considered FL-UAV and FML networks. We now summarize related works in these areas and compare methodology design features between our paper and related works. 

\textit{1) FL-UAV:} Most previous research in this area has primarily concentrated on communication and/or computation aspects \cite{6,9}. In \cite{5}, the authors proposed a FL-aided image classification approach for UAV-aided exploration scenarios, enhancing classification accuracy while reducing communication costs and computational complexity. The work in \cite{11} contributed a UAV-empowered wireless power transfer solution for sustainable FL-based wireless networks, optimizing power efficiency through a joint optimization algorithm that reduces UAV transmit power. The authors in \cite{zeng2020federated} proposed a distributed FL framework for UAV swarms that optimizes convergence by jointly allocating power and scheduling, reducing communication rounds while considering wireless factors and energy consumption. In \cite{pham2022energy}, an energy-efficient framework for Federated Learning (FL) is introduced, utilizing UAV-assisted wireless power transmission to optimize resource distribution. This approach aims to reduce overall energy usage while improving the sustainability of FL networks. In \cite{fu2023federated}, a UAV-assisted FL system is proposed to minimize training time through optimization of device scheduling, UAV path, and energy constraints. Similarly, \cite{nguyen2022fedfog} presents an FL algorithm dedicated to wireless fog-cloud systems, where the authors concentrate on training time and global loss. Data sensing has become a key focus in FL research, garnering increasing attention recently. In \cite{3}, a combined resource distribution approach for federated edge learning was introduced, optimizing human motion recognition by effectively managing sensing, computation, and communication resources. The work in \cite{12} introduced a multi-task deep learning framework for optimizing sensing, communication, and computation resources using multi-objective optimization. In \cite{10}, an optimization scheme for data sensing over UAV networks was presented, improving energy utilization by tackling several network parameters together. In \cite{din2025federated}, a unified FL framework was developed for UAV-enabled Internet of Things networks, showing improved accuracy, resilience under attacks, and scalability in large deployments. The authors in \cite{li2025exploring} proposed a hierarchical FL framework to improve robustness in UAV-based object detection missions, leveraging three-dimensional graph-based clustering, intragroup backups, and adaptive server selection. \textit{However, these works do not address latency minimization in such networks, a crucial factor for real-time applications where timely data acquisition and processing are essential. Optimally coordinating data sensing, computation, and communication steps is vital to enhance response times in time-sensitive scenarios.}

\textit{2) FML:} FML in wireless networks has become an important research topic due to its ability to integrate data from multiple sources. In \cite{zhao2022multimodal}, a multimodal, semi-supervised FL framework was proposed to enhance classification by training local autoencoders on different data types and using auxiliary labeled data for aggregation. The study in \cite{yin2024aggregation} introduced a parameter scheduling approach for wireless personalized FML, improving both personalization and communication efficiency through learning-based aggregation and modality-specific scheduling. In \cite{tun2025resource}, a resource-efficient layer-wise and progressive training strategy was proposed to reduce memory, computation, and communication costs in FML systems. The work in \cite{gao2025multimodal} introduced a multi-view domain fusion framework with global logit alignment and local angular margin to address modality-induced data heterogeneity in FL. \textit{However, no FML frameworks have yet been developed specifically for UAV networks.} 

In spite of these advancements, \textit{the problem of latency minimization in UAV-enabled FML systems remains underexplored}. Considering the limited computational capacity and battery life of UAVs, optimizing the round-trip ML model training latency in relation to UAV resources, such as transmit power and computational frequency, is essential for achieving efficient and timely FML. 


\begin{figure}[t!]
    \centering
    \includegraphics[width=0.99\linewidth]{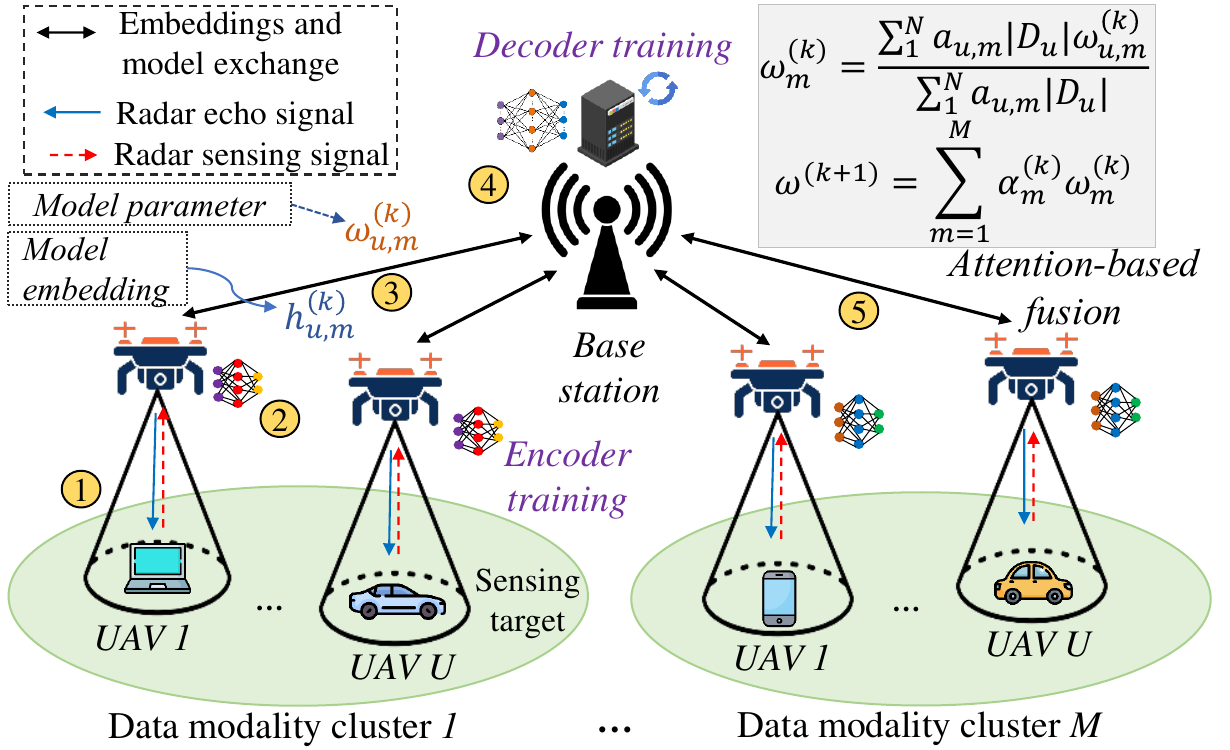} 
    \captionsetup{font=footnotesize}
    \caption{\footnotesize \textcolor{black}{Proposed FML framework over UAV networks with five key steps: (1) UAVs sense data from the ground object, (2) train local models on the sensed data, (3) upload embeddings and models to the Base Station (BS), (4) the BS trains a decoder model using concatenated embeddings, and (5) the BS aggregates local models to create a unified global model, sending it back to the UAVs.}}
    \label{Fig:Overview}
    \vspace{-15pt}
\end{figure}

\vspace{-5pt}

\subsection{Motivations and Key Contributions}
Inspired by the limitations in existing literature, our paper makes the following contributions:
\begin{itemize}
    \item We propose a novel UAV-assisted FML framework in which distributed UAVs work together to train a shared ML model, with a BS serving as the central server. To improve model accuracy and enhance generalization, we incorporate multi-modal data sensing by UAVs, tackling the limitations of single-modal data and offering a more comprehensive understanding of the environment. Additionally, we provide a comprehensive convergence analysis of our proposed UAV-enabled FML framework under a non-convex loss function scenario.
    \item We define a new latency minimization problem for the FML frmawork that considers several crucial parameters, such as: UAV's sensing scheduling, power control, trajectory, resource allocation, and BS resource allocation. As the problem is computationally intractable for traditional convex solvers in its current form, we introduce an iterative optimization approach that combines block coordinate descent (BCD) and successive convex approximation (SCA) techniques to find optimal solutions.
    \item We perform extensive simulations to assess the performance of our UAV-enabled FML framework and the joint optimization scheme. The results demonstrate that our proposed FML framework outperforms baseline methods in model loss and accuracy convergence, in both idependent and identically distributed (IID) and non-IID data settings. Additionally, our approach reduces system latency by up to $42.49\%$, compared to benchmark schemes.
\end{itemize}


\section{System Model} 
\subsection{FML Model Formulation over UAV Networks}
We consider a FML framework over UAV networks as illustrated in Fig.~\ref{Fig:Overview}. In our system model, a base station (BS) orchestrates the FML process where distributed UAVs from various modality clusters work together to train a shared ML model. Our system incorporates $M$ data modalities and the set of modalities is denoted as $\mathcal{M} = \{1,2,\dots,M\}$. Here, $m$ refers to a specific modality while we mention that $M < U$. In each modality cluster $m$, the set of UAVs is represented by $\mathcal{U} = \{1,2,\dots,U\}$. Each UAV $u$ is equipped with a single-antenna transceiver that can alternate between sensing and communication modes as required. This mode switching is carried out in a time-division manner using a shared radio-frequency interface. \cite{zhang2022accelerating}. UAV $u$ is assumed to collect (by sensing) data $\mathcal{D}_{u}$ of size $D_{u} \triangleq |\mathcal{D}_{u}|$. The union of datasets gathered all UAVs, referred to as the global dataset, is denoted as $\mathcal{D} = \cup_{u \in \mathcal{U}} \mathcal{D}_{u}$, with size $D \triangleq |\mathcal{D}|$. The set of global communication rounds in FML is represented as $\mathcal{K} = \{1,2,\dots,K\}$. 
The procedure of FML model training during each global round $k \in \mathcal{K}$ is summarized as follows:
\begin{enumerate}
    \item Each UAV involved in the process senses data from the selected ground object within its coverage region. 
    \item Then, each UAV trains its local model using the gathered data, extracting embeddings and model parameters upon completion of the training.
    \item Subsequently, the UAV sends its local embeddings and model parameters to the server for aggregation.
    \item Upon receiving the embeddings and model parameters from all participating UAVs, the BS aggregates them for each modality group. The BS then merges (concatenation) the aggregated embeddings from all modality groups, creating a unified embedding that is passed to the decoder for tasks like classification.
    \item Finally, the BS sends the aggregated model parameters for each modality group back to the respective UAVs to begin the next round of training.  
\end{enumerate}


We define a 3D Cartesian coordinate system where the ground targets and the base station (BS) remain stationary. The BS is located at the origin $(0,0,0)$. UAV $u$ begins from a point near the targets, hovers above them to sense and collect data from the chosen target, and trains its local model. Afterward, the UAV flies towards the BS to transmit the local embeddings and model parameters during its flight time $T_{\text{flight}}$. The UAV maintains a constant altitude of $H > 0$ above the ground. This communication phase is divided into $T$ equal time slots, with each slot having a duration of $\delta_{t} = \frac{T_{\text{flight}}}{T}$. To ensure the UAV's position remains nearly constant within each slot, the slot duration is selected to be small enough. As a result, the UAV's horizontal position over time is denoted by $q_{u}[t] \triangleq (x_{u}^{(k)}[t],y_{u}^{(k)}[t])$, where $t \in \mathcal{T} \triangleq \{1,2,\dots,T\}$. We assume UAV $u$ starts its journey at position $q_{I} = [x_{I},y_{I}]$ and reaches the final position $q_{F} = [x_{F},y_{F}]$ during its total flight time $T_{\text{flight}}$. 

In our framework, the maximum UAV velocity is denoted by $V_{\text{max}}$. Therefore, the UAV trajectory has to abide by the constraint $(x_{u}^{(k)}[t+1] - x_{u}^{(k)}[t])^2 + (y_{u}^{(k)}[t+1] - y_{u}^{(k)}[t])^2 \leq (V_{\max} \delta_{t})^2$. This restricts the maximum movement of the UAV in consecutive time slots.
We assume that the single-antenna sensing targets cannot be directly served by the BS because of the blockage of surrounding obstacles. Additionally, it is assumed that all communication links between the UAVs and the BS, as well as those between the UAVs and their sensing targets, are line-of-sight (LoS) channels. Therefore, the LoS channel gain between UAV $u$ and BS at time slot $t$ abides by free space pathloss model, expresses as $g_{u,\text{BS}}^{(k)}[t] = \frac{\beta_{0}}{d_{u,\text{BS}}^{(k)}[t]^2}$. Here, $\beta_{0}$ is the channel gain at reference distance $d_{0}=1 m$, and $d_{u,\text{BS}}^{(k)}[t]$ denotes distance between UAV $u$ and the BS at time slot $t$.

In this study, we employ a FML framework with an encoder-decoder architecture, as illustrated in Fig.~\ref{Fig:Overview-Structure}. It is important to note that encoders and decoders are ML models designed for specific functions. Each UAV in every modality cluster is equipped with an encoder (a feature extractor), while the BS operates a decoder (a classifier) to generate the final training results, such as classification outcomes \cite{zhao2022multimodal}. 
Specifically, the encoders extract features from the single-modal data sets owned by each UAV, while the decoders perform classification tasks at the server. For each UAV $u$, each data point $d \triangleq (X,y) \in \mathcal{D}_{u}$ consists of a feature vector $X$ and a label $y$, where $X = \{x_{m}\}$ represents the set of features corresponding to modality $m$ that UAV $u$ holds. For UAV $u$ with data of modality $m$, this data is passed through the corresponding single modal encoder $\boldsymbol{w}_{u,m}(.)$ to generate feature embeddings, denoted as $h_{u,m}^{(k)} = \boldsymbol{w}_{u,m}^{(k)}(x_{u,m}^{(k)})$.
The BS collects embeddings of each modality $m$ from $u$ UAVs and performs modality-based separate aggregation as 
\begin{equation}
    h_{m}^{(k)} = \frac{1}{U} \sum_{u=1}^{U} {h_{u,m}^{(k)}}.
\end{equation}
After aggregating the embeddings from all available modalities in the system, the BS concatenates them to form a unified embedding, given by
\begin{equation}
    h^{(k)} = {h_{1}^{(k)} \oplus h_{2}^{(k)} \oplus h_{3}^{(k)} \dots h_{M}^{(k)}}.
\end{equation}
This concatenated multimodal feature embedding is then input into the BS decoder $\boldsymbol{w}_{\text{BS}}(.)$ to generate a prediction $\hat{y}$, expressed as $\hat{y} = \boldsymbol{w}_{\text{BS}}(h)$.
Each UAV $u$ participates in $J$ iterations of SGD and $K$ global communication rounds. Specifically, in global round $k$, UAV $u$ performs $J$ SGD iterations before sending its updated model to the server. The final model after local training at UAV $u$ is denoted as $\boldsymbol{w}_{u,m}^{(k),J}$. Once the local model training is completed, the BS aggregates the received model parameters for each modality individually as
\begin{equation}
    \boldsymbol{w}_{m}^{(k)} = \frac{\sum_{u \in \mathcal{U}} a_{u,m} |\mathcal{D}_{u}| \boldsymbol{w}_{u,m}^{(k),J}}{\sum_{u \in \mathcal{U}} a_{u,m} |\mathcal{D}_{u}|}, \forall{m \in \mathcal{M}},
\end{equation}
where, $a_{u,m}, \forall{u \in \mathcal{U}, m \in \mathcal{M}}$ is a binary indicator which equals 1 if UAV $u$ has access to modality $m$, and 0 otherwise. 

\textcolor{black}{It is to note that these parameters encapsulate the learned patterns from all participating UAVs for that particular modality. After conducting model parameters aggregation, a set of high-level features is extracted from these aggregated parameters. Let us denote the extracted high-level features from aggregated model parameters $\boldsymbol{w}_{m}^{(k)}$ of modality $m$ as $z_{m}^{(k)}$. To extract these features, a subset of the data is used and is passed through the model configured with the aggregated parameters $\boldsymbol{w}_{m}^{(k)}$, i.e. the encoder model for modality $m$. The output $z_{m}^{(k)}$ from this operation will be a set of high-level features, effectively capturing the essential patterns and characteristics of the data relevant to that modality. Once the high-level features for each modality are extracted, the attention scores are computed. These scores determine the importance of each modality’s contribution to the global model. We denote the attention scoring function as $f$ which takes the high-level features as input and outputs a raw score. To turn the raw scores into a usable format that reflects probabilities (i.e., how much each modality should contribute), a softmax function is applied as
\begin{equation}
    \alpha_{m}^{(k)} = \text{softmax}(f(z_{m}^{(k)})),
\end{equation}
where $\alpha_{m}^{(k)}$ is the attention score for modality $m$ during global round $k$. The softmax function ensures that all the attention scores sum up to 1, making them effectively a distribution over modalities. Using the attention scores, BS performs a weighted averaging of the aggregated model parameters from all the modalities as
\begin{equation}
    \boldsymbol{w}_{g}^{(k+1)} = \frac{\sum_{m=1}^{M} \alpha_{m}^{(k)} \boldsymbol{w}_{m}^{(k)}}{\sum_{m=1}^{M} \alpha_{m}^{(k)}}.
\end{equation}
The BS transmits the aggregated model parameters for each modality group to the corresponding participating UAVs for use in the next training round.
\begin{equation}
    \boldsymbol{w}_{u,m}^{(k+1),0} \leftarrow \boldsymbol{w}_{m}^{(k)}, \forall{u \in \mathcal{U},m \in \mathcal{M}}
\end{equation}} 
\begin{figure}[t]
    \centering
    \includegraphics[width=0.99\linewidth]{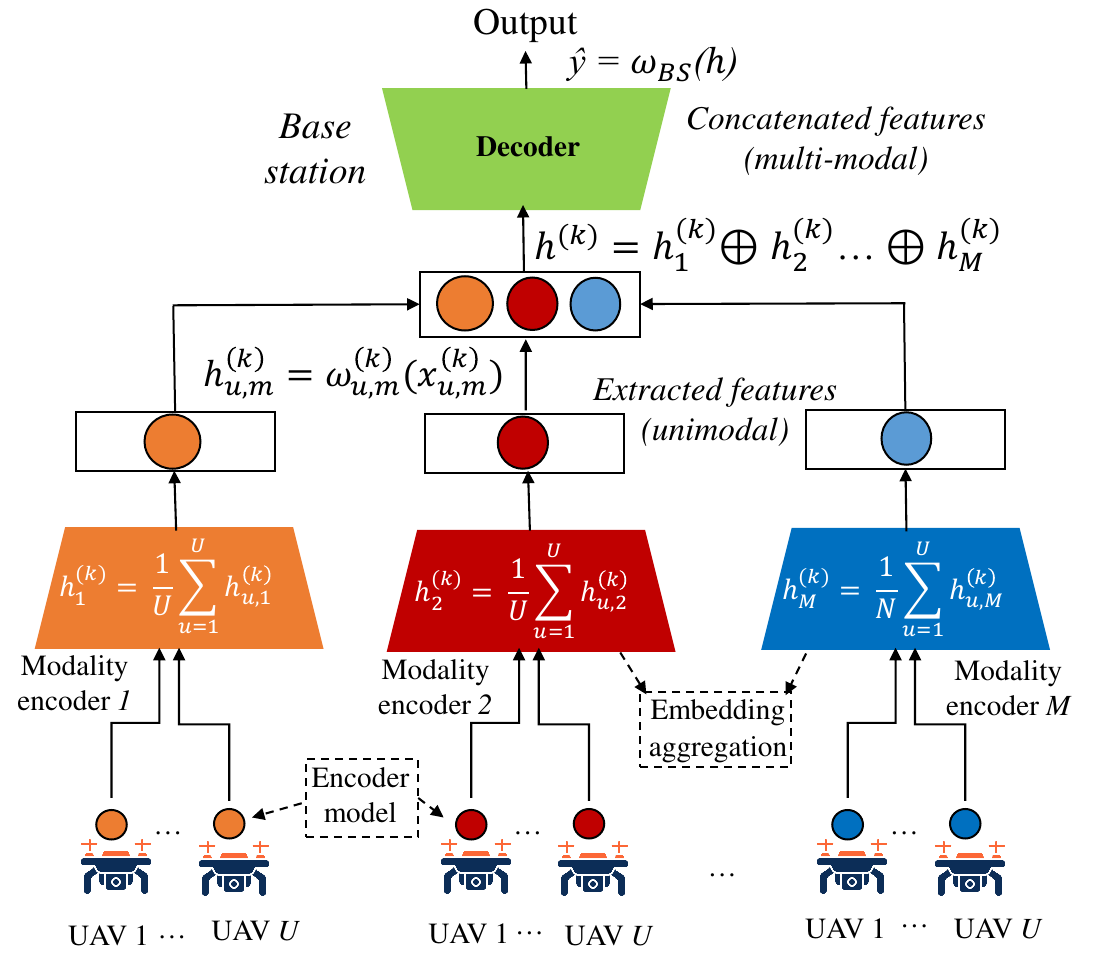} 
    \caption{The encoder-decoder architecture in the proposed FML framework. }
    \label{Fig:Overview-Structure}
    \vspace{-15pt}
\end{figure}

We calculate the local loss function of UAV $u$ holding data of modality $m$ as 
\begin{equation}
    \mathcal{L}_{u,m}(\boldsymbol{w}_{u,m}^{(k)};\mathcal{D}_{u}) = \frac{1}{|\mathcal{D}_{u}|} \sum_{d \in \mathcal{D}_{u}} \ell(\boldsymbol{w}_{u,m}^{(k)};d),
\end{equation}
where $\ell(\boldsymbol{w}_{u,m}^{(k)};d)$ represents the loss function associated with ML model calculated on data point $d$. As a result, the global loss function is formulated as
\begin{equation}
    \mathcal{L}_{m}(\boldsymbol{w}) \triangleq \sum_{u \in \mathcal{U}} \frac{|\mathcal{D}_{u}|}{|\mathcal{D}|} \mathcal{L}_{u,m}(\boldsymbol{w}_{u,m}^{(k)};\mathcal{D}_{u}).
\end{equation}
The local ML model training at each UAV is performed through multiple minibatch SGD iterations. For a local model at UAV $u$ of modality $m$, during the $j^{\text{th}}$ SGD iteration in the $k^{\text{th}}$ global round (i.e., $\boldsymbol{w}_{u,m}^{(k), j}$), the subsequent local model is updated as 
\begin{equation}
    \boldsymbol{w}_{u,m}^{(k),\text{j+1}} = \boldsymbol{w}_{u,m}^{(k),J} - \eta^{k,j} \Tilde{\nabla} \mathcal{L}_{u,m}(\boldsymbol{w}_{u,m}^{(k),J};\beta_{u}^{k,j}),
\end{equation}
where 
\begin{equation}
    \Tilde{\nabla} \mathcal{L}_{n,m}(\boldsymbol{w}_{u,m}^{(k),J};\beta_{u}^{k,j}) \triangleq \frac{1}{\beta_{u}^{k,j}} \sum_{d \in \beta_{u}^{k,j}} \nabla \ell(\boldsymbol{w}_{u,m}^{(k),J};d).
\end{equation}
Here, $\beta_{u}^{k,j}$ represents a mini-batch of data randomly sampled from $\mathcal{D}_{u}$, and $\eta^{k,j}$ denotes the learning rate for SGD.
For clarity, the notation table is provided below to summarize the key symbols used throughout the paper.
\begin{table}[t]
\centering
\caption{Summary of Notations}
\begin{tabular}{ll}
\hline
\textbf{Symbol} & \textbf{Description} \\
\hline
$M$ & Number of data modalities in the network \\
$U$ & Number of UAVs in a modality cluster \\
$K$ & Number of global communication rounds \\
$J$ & Number of local iterations \\
$T_{\text{flight}}$ & UAV's flight time \\
$T$ & Time slots of the UAV communication phase \\
$\delta_{t}$ & Duration of a time slot \\
$V_{\text{max}}$ & Maximum UAV velocity \\
$g_{u,\text{BS}}^{(k)}[t]$ & Channel gain for UAV-BS link \\
$d_{u,\text{BS}}^{(k)}[t]$ & Distance between UAV and the BS \\
$\beta_{0}$ & Channel gain at reference distance \\
$\boldsymbol{w}_{u}^{k,j}$ & Local model parameters of UAV \\
$g_{u}^{k,j}$ & Full gradient of UAV \\
$\Tilde{g}_{u}^{k,j}$ & Stochastic gradient of UAV \\
$x_{c,u}^{(k)}$ & Sensing scheduling of UAV \\
$D_{u}^{(k)}$ & Number of data samples sensed by UAV \\
$\mathbf{T}_{\text{se},u}^{(k)}$ & Data sensing time of UAV \\
$E_{\text{se},u}^{(k)}$ & Data sensing energy of UAV \\
$p_{\text{se}, u}^{(k)}$ & Sensing transmit power of UAV \\
$\mathbf{T}_{\text{em-cm},u}^{(k)}$ & Local embeddings uploading time of UAV \\
$\mathbf{T}_{\text{ml-cm},u}^{(k)}$ & Local model parameters uploading time of UAV \\
$E_{\text{em-cm},u}^{(k)}$ & Local embeddings uploading energy of UAV \\
$E_{\text{ml-cm},u}^{(k)}$ & Local model parameters uploading energy of UAV \\
$J'$ & Number of iterations for server-side training \\
$C_{\text{BS}}^{(k)}$ & CPU cycles per sample during server training \\
$f_{\text{BS}}^{(k)}$ & CPU processing rate of the BS \\
$R_{\text{BS}}^{(k)}$ & BS-UAV Downlink rate \\
$\mathbf{T}_{\text{dl},u}^{(k)}$ & Global model downloading time of UAV \\
$p_{\text{cm},\text{BS}}^{(k)}$ & Communication transmit power of BS \\
$p_{\text{cm},u}^{(k)}[t]$ & Communication transmit power of UAV \\
$f_{u}^{(k)}$ & CPU computation capability of UAV \\
\hline
\end{tabular}
\end{table}

\subsection{\textcolor{black}{Convergence Analysis for Proposed Multimodal FL Framework}}
In the FML framework, the federated model training is performed across modality clusters. UAVs within a modality cluster exchange ML models with the BS, where the global model aggregation is executed for each cluster. However, to facilitate our convergence analysis, we consider an attention-based fusion for federated averaging across all the modalities where aggregated model parameters from all the modality clusters are fused into one unified global model. This section is dedicated to the convergence analysis of the proposed FML algorithm in the scenario where all UAVs participate. \textit{Our findings reveal that the convergence rate is dependent on the total number of iterations, the number of total UAVs in each modality cluster, and the number of data modalities present in the system.}

\subsubsection{Notation and Definition}
For the convergence analysis, we concentrate on the following optimization problem:
\begin{align}
    \min_{\boldsymbol{w}_{g}} f(\boldsymbol{w}_{g}) & \triangleq \sum_{m=1}^{M} \alpha_{m} \sum_{u=1}^{U} f_{n,m}(\boldsymbol{w}_{u,m}), \label{eqn23}
\end{align}
where $f(\boldsymbol{w}_{g})$ is the global objective function. First, we find the convergence upper bound for a modality cluster $m$. Then we expand our analysis to find the convergence upper bound for the attention-based fused global model. For simplicity, we temporarily omit the notation $m$ in our discussion, i.e., $\boldsymbol{w}_{u,m}$ is now written as $\boldsymbol{w}_{u}$. Then we bring the notation back into our analysis later. In this multimodal federated framework, within a cluster of modality $m$, it is assumed that each UAV $u$ trains its local model on dataset $\mathcal{S}_{u}$ containing $S_{u}$ data points sampled from the local distribution $\mathcal{D}_{u}$. Since the local datasets are generated from different distributions, we carefully consider the heterogeneity of these distributions while analyzing the convergence of FML. We define $g_{u} = \frac{1}{|\mathcal{S}_{u}|} \nabla f_{u}(\boldsymbol{w})\stackrel{\triangle}{=} \frac{1}{|\mathcal{S}_{u}|} \nabla f(\boldsymbol{w};\mathcal{S}_{u})$, where $f(\boldsymbol{w};\mathcal{S}_{u})$ represents the full gradient. Moreover, we denote the stochastic gradient as $\Tilde{g_{u}} \stackrel{\triangle}{=} \frac{1}{B} \nabla f(\boldsymbol{w};\xi_{u})$, where $\xi_{u} \subseteq \mathcal{S}_{u}$ is a uniformly sampled mini-batch with $|\xi_{u}| = B$. The corresponding quantities evaluated at device $u$'s local solution $\boldsymbol{w}_{u}^{k,j}$ during local iteration $j$ of the $k^{\text{th}}$ global round are denoted by $g_{u}^{k,j}$ for the full gradient and $\Tilde{g}_{u}^{k,j}$ for the stochastic gradient. We also define the following notations:
\begin{align}
    \boldsymbol{w}^{k,j} = [\boldsymbol{w}_{1}^{k,j}, \boldsymbol{w}_{2}^{k,j}, \dots, \boldsymbol{w}_{U}^{k,j}],
\end{align}
\begin{align}
    \xi^{k,j} = [\xi_{1}^{k,j}, \xi_{2}^{k,j}, \dots, \xi_{U}^{k,j}],
\end{align}
in order to represent the set of local solutions and sampled mini-batches associated with the devices during local iteration $j$ at $k^{\text{th}}$ global round, respectively. The following notations will be useful for the convergence analysis of the FML framework:
$
    \bar{\boldsymbol{w}}^{k,j} \stackrel{\triangle}{=} \frac{1}{U} \sum_{u \in \mathcal{U}} \boldsymbol{w}_{u}^{k,j},
$
$
    \Tilde{g}^{k,j} \stackrel{\triangle}{=} \frac{1}{U} \sum_{u \in \mathcal{U}} \Tilde{g}_{u}^{k,j},
$
$
    g^{k,j} \stackrel{\triangle}{=} \frac{1}{U} \sum_{u \in \mathcal{U}} g_{u}^{k,j}.
$
Thus, the local SGD update at device $u$ is followed as
$
    \boldsymbol{w}_{n}^{k,j+1} = \boldsymbol{w}_{u}^{k,j} - \eta_{k} \Tilde{g}_{u}^{k,j}
$
It is apparent that 
\begin{align}
    \bar{\boldsymbol{w}}^{k,j+1} = \bar{\boldsymbol{w}}^{k,j} - \eta_{k} \Tilde{g}^{k,j}. \label{eqn30}
\end{align}
It is to be mentioned that $\mathbb{E} \Tilde{g}^{k,j} = g^{k,j}$, where $\mathbb{E}$ represents function's expectation. In the subsequent analysis, we assume that $\lambda$ represents an upper limit on the gradient variability across the local objectives, i.e.,
\begin{align}
    \frac{\sum_{u=1}^{U} ||g_{u}^{k,j}||_{2}^{2}}{||\sum_{u=1}^{U}g_{u}^{k,j}||_{2}^{2}} \leq \lambda. \label{eqn31}
\end{align}
In the following subsection, we delineate the foundational assumptions that underlie our convergence analysis.

\subsubsection{Assumptions}
\begin{assumption}[Smoothness and Lower Bound] \label{Assumption1}
\textit{The local objective function $f_{n}(.)$ for device $u$ is differentiable for $1 \leq u \leq U$ and satisfies the $L-smooth$ property, i.e., $||\nabla f_{u}(\mathbf{u}) - \nabla f_{u}(\mathbf{v})|| \leq L||\mathbf{u}-\mathbf{v}||, \forall{\mathbf{u},\mathbf{v}} \in \mathbb{R}^{d}$.}
\end{assumption}
\begin{assumption}[$\mu$-Polyak-Lojasiewicz (PL) Condition] \label{Assumption2}
\textit{The global objective function $f(.)$ is differentiable and satisfies the Polyak-Lojasiewicz (PL) condition with constant $\mu$, i.e., $\frac{1}{2} ||\nabla f(\boldsymbol{w})||_{2}^{2} \geq \mu (f(\boldsymbol{w})-f({\boldsymbol{w}}^{*}))$ holds $\forall \boldsymbol{w} \in \mathbb{R}^{d}$, where ${\boldsymbol{w}}^{*}$ is the optimal global solution.}
\end{assumption}
\begin{assumption}[Bounded Local Variance] \label{Assumption3}
\textit{For every local dataset $S_{u}$, $u = 1, 2, \dots, U$, we can sample an independent mini-batch $\xi_{u} \subseteq \mathcal{S}_{u}$ with $|\xi_{u}| = B$ and compute an unbiased stochastic gradient $\Tilde{g_{u}} = \frac{1}{B} \nabla f(\boldsymbol{w};\xi_{u})$, $\mathbb{E}[\Tilde{g_{u}}] = g_{u} = \frac{1}{|\mathcal{S}_{u}|} \nabla f(\boldsymbol{w};\mathcal{S}_{u})$ with the variance bounded as
\begin{align}
    \mathbb{E}[||\Tilde{g_{u}} - g_{u}||^{2}] \leq C_{1} ||g_{u}||^{2} + \frac{\sigma^{2}}{B}. 
\end{align}
where $C_1$ is a non-negative constant that is inversely related to the mini-batch size, and $\sigma$ is another constant that governs the variance bound.}
\end{assumption}
\noindent
Based on the update rule in \eqref{eqn30} and the assumption of L-smoothness for the objective function, the following inequality holds:
\begin{align}
    f(\bar{\boldsymbol{w}}^{k,j+1}) - f(\bar{\boldsymbol{w}}^{k,j}) &\leq - \eta_{k} \langle \nabla f(\bar{\boldsymbol{w}}^{k,j}), \Tilde{g}^{k,j} \rangle \nonumber \\
    &+ \frac{\eta_{k}^{2}  L}{2} ||\Tilde{g}^{k,j}||^{2}. \label{eqn37new}
\end{align}
Taking the expected value of both sides of the inequality in \eqref{eqn37new} gives us 
\begin{align}
    \mathbb{E}[f(\bar{\boldsymbol{w}}^{k,j+1}) - f(\bar{\boldsymbol{w}}^{k,j})] &\leq -\eta_{k} \mathbb{E}[\langle \nabla f(\bar{\boldsymbol{w}}^{k,j}), \Tilde{g}^{k,j} \rangle] \nonumber \\
    &+ \frac{\eta_{k}^{2}  L}{2} \mathbb{E}[||\Tilde{g}^{k,j}||^{2}]
\end{align}
By taking the average for all the local and global iterations, we get
\begin{align}
    &\frac{1}{KJ} \sum_{k=1}^{K} \sum_{j=1}^{J} \mathbb{E}[f(\bar{\boldsymbol{w}}^{k,j+1}) - f(\bar{\boldsymbol{w}}^{k,j})] \nonumber \\
    & \hspace{6em}\leq \frac{1}{KJ} \sum_{k=1}^{K} \sum_{j=1}^{J} (-\eta_{k} \mathbb{E}[\langle \nabla f(\bar{\boldsymbol{w}}^{k,j}), \Tilde{g}^{k,j} \rangle]) \nonumber \\
    & \hspace{6em}+ \frac{1}{KJ} \sum_{k=1}^{K} \sum_{j=1}^{J} \frac{\eta_{k}^{2}  L}{2} \mathbb{E}[||\Tilde{g}^{k,j}||^{2}]. \label{eqn35}
\end{align}
Moving forward, we now systematically determine bounds for each term appearing on the right-hand side of \eqref{eqn35}. Specifically, Lemma~\ref{lemma1} is utilized to ascertain a bound for the first term in this equation. Subsequently, Lemma~\ref{lemma3} is employed to derive a bound for the second term. Additionally, Lemma~\ref{lemma2} focuses on a term that originates from the analysis in Lemma~\ref{lemma1}—particularly, it addresses the final term delineated in Lemma~\ref{lemma1}, providing its bound to further improve our understanding of the overall equation's dynamics.
\noindent

\subsubsection{Convergence Rates}
\noindent We next present several lemmas that are utilized in deriving the main result.

\begin{lemma} \label{lemma1}
Let Assumption \ref{Assumption1} hold, the expected value of the inner product between the stochastic gradient and full gradient is limited by
\begin{align}
    & -\eta_{k} \mathbb{E}\bigg[\langle \nabla f(\bar{\boldsymbol{w}}^{k,j}), \Tilde{g}^{k,j} \rangle\bigg] \leq -\frac{\eta_{k}}{2}||\nabla f(\bar{\boldsymbol{w}}^{k,j})||^2 \nonumber \\
    & - \frac{\eta_{k}}{2}||\sum_{u=1}^{U}\nabla f_{u}(\boldsymbol{w}_{u}^{k,j})||^{2} + \frac{\eta_{k} L^{2}}{2} \sum_{u=1}^{U} ||\bar{\boldsymbol{w}}^{k,j} - \boldsymbol{w}_{u}^{k,j}||^{2}.
\end{align}
\end{lemma}
\begin{proof}
See Section \ref{prooflemma1}. \renewcommand{\qedsymbol}{}
\end{proof}
\begin{lemma} \label{lemma2}
Provided that Assumption \ref{Assumption3} is fulfilled, the expected upper bound of the divergence of $\boldsymbol{w}_{u}^{k,j}$ is given as 
\begin{align}
    & \frac{1}{KJ} \sum_{k=1}^{K} \sum_{j=1}^{J} \sum_{u=1}^{U} \bigg[\mathbb{E}||\bar{\boldsymbol{w}}^{k,j} - \boldsymbol{w}_{u}^{k,j}||\bigg] \nonumber \\ 
    & \leq \frac{(2C_{1} + J(J+1))}{KJ} \eta_{k}^{2} \frac{U+1}{U} \frac{1}{KJ} \sum_{k=1}^{K} \sum_{j=1}^{J} \sum_{u=1}^{U} ||g_{u}^{k,j}||^{2} \nonumber \\ 
    & + \frac{\eta_{k}^{2} (U+1)(J+1)\sigma^{2}}{UB} \nonumber \\
    & \leq \frac{\lambda \eta_{K}^{2} (2C_{1} + J(J+1))}{KJ} \frac{U+1}{U} \frac{1}{KJ} \sum_{k=1}^{K} \sum_{j=1}^{J} \sum_{u=1}^{U} ||g_{u}^{k,j}||^{2} \nonumber \\
    & + \frac{\eta_{k}^{2} KJ (U+1)(J+1)\sigma^{2}}{UB}.
\end{align}
\end{lemma}
\noindent
\begin{proof}
See Section \ref{prooflemma2}. \renewcommand{\qedsymbol}{}
\end{proof}
\begin{lemma} \label{lemma3}
Under Assumption \ref{Assumption3}, the expected upper bound of $\mathbb{E}[||\Tilde{g}^{k,j}||^2]$ is expressed as
\begin{align}
    \mathbb{E}\bigg[||\Tilde{g}^{k,j}||^2\bigg] &\leq \bigg(\frac{C_{1}}{U}+1\bigg) \bigg[\sum_{u=1}^{U}||\nabla f_{u}(\boldsymbol{w}_{u}^{k,j})||^{2}\bigg] + \frac{\sigma^{2}}{UB} \nonumber \\ 
    &\leq \lambda \bigg(\frac{C_{1}}{U}+1\bigg) \bigg[\sum_{u=1}^{U}||\nabla f_{u}(\boldsymbol{w}_{u}^{k,j})||^{2}\bigg] + \frac{\sigma^{2}}{UB}.
\end{align}
\end{lemma}
\begin{proof}
See Section \ref{prooflemma3}. \renewcommand{\qedsymbol}{}
\end{proof}

\noindent
\textit{\textbf{Theorem 1.}}
\textit{Let Assumptions \ref{Assumption1}, \ref{Assumption2}, \ref{Assumption3} hold, then the upper bound of the convergence rate of the global model training considering full device participation after $K$ global rounds satisfies}
\begin{align}
    & \frac{1}{KJ} \sum_{k=1}^{K} \sum_{j=1}^{J} \mathbb{E}||\nabla f(\bar{\boldsymbol{w}}^{k,j})||^{2} \leq \frac{2 [f(\bar{\boldsymbol{w}}_{1}^{0}) - f^{*}]}{\eta_{k} KJ} + \frac{L \eta_{k} \sigma^{2}}{UB} \nonumber \\
    & + \frac{2 \eta_{k}^{2} \sigma^{2} L^{2} (J+1)}{B} \bigg(1+\frac{1}{U}\bigg). \label{eqn42new}
\end{align}
\begin{proof}
See Section \ref{prooftheorem1}. \renewcommand{\qedsymbol}{}
\end{proof}

\begin{remark}
The convergence upper bound derived in Theorem 1 reveals several important insights about the behavior of the proposed FML algorithm under full device participation. Specifically, the bound in \eqref{eqn42new} shows that the expected gradient norm decreases over time, ensuring convergence of the global model. The first term in the bound, $\frac{2 [f(\bar{\boldsymbol{w}}_{1}^{0}) - f^{*}]}{\eta_{k} KJ}$, indicates that increasing the number of global rounds $K$ and local iterations $J$ improves convergence by reducing the gap between the current and optimal objective values. The second and third terms reflect the impact of stochastic noise in gradient estimation, where higher mini-batch size $B$ and a larger number of participating UAVs $U$ reduce the variance and thus improve convergence. In particular, the presence of $1/U$ in both terms demonstrates that involving more UAVs in training helps smooth out local variations and noise, resulting in a more stable and efficient training process. Overall, this result confirms that the convergence of the proposed algorithm is positively influenced by the number of global rounds, local iterations, mini-batch size, and UAV count. These are key factors that can be tuned to balance training efficiency and stability in practical deployments.
\end{remark}

 \subsection{Detailed Proofs for Convergence Analysis}
 In this section, we present proofs of lemmas and theorems used the above section.
\subsubsection{Proof of Lemma \ref{lemma1}} \label{prooflemma1}
As mentioned before, let $\mathcal{U} = \{1, 2, \dots, U\}$ denote the set of UAVs for modality cluster $m$, and let $\Tilde{g}^{k,j} = \frac{1}{U} \sum_{u \in \mathcal{U}} \Tilde{g}_{u}^{k,j}$ represent the average of their local stochastic gradients at local iteration $j$ during global round $k$. We have
\begin{align}
    &- \mathbb{E}_{\{\xi_{1}^{k,j}, \dots, \xi_{U}^{k,j} | \boldsymbol{w}_{1}^{k,j}, \dots, \boldsymbol{w}_{U}^{k,j}\}} \nonumber \\
    &\hspace{7em} \mathbb{E}_{\{1, 2, \dots\, U \}\in \mathcal{U}} \bigg[\langle \nabla f(\bar{\boldsymbol{w}}^{k,j}), \Tilde{g}^{k,j} \rangle\bigg] \nonumber \\
    &= - \mathbb{E}_{\{\xi_{1}^{k,j}, \dots, \xi_{U}^{k,j} | \boldsymbol{w}_{1}^{k,j}, \dots, \boldsymbol{w}_{U}^{k,j}\}} \nonumber \\
    & \hspace{7em} \mathbb{E}_{\{1, 2, \dots\, U \}\in \mathcal{U}} \bigg[\langle \nabla f(\bar{\boldsymbol{w}}^{k,j}), \frac{1}{U}\sum_{u \in \mathcal{U}}\Tilde{g}_{u}^{k,j} \rangle\bigg] \nonumber \\
    & \hspace{6em} + || \sum_{u=0}^{U} \bigg(\nabla f_{u}(\bar{\boldsymbol{w}}^{k,j}) - \nabla f_{u}(\boldsymbol{w}_{u}^{k,j}) \bigg) ||_{2}^{2}\bigg] \nonumber \\
    &\stackrel{\textcircled{\footnotesize 3}}{\leq} \frac{1}{2} \bigg[ 
    -||\nabla f(\bar{\boldsymbol{w}}^{k,j})||_{2}^{2} - || \sum_{u=0}^{U} \nabla f_{u}(\boldsymbol{w}_{u}^{k,j})||_{2}^{2} \nonumber \\
    & \hspace{8em} + \sum_{u=0}^{U} || \nabla f_{n}(\bar{\boldsymbol{w}}^{k,j}) - \nabla f_{u}(\boldsymbol{w}_{u}^{k,j}) ||_{2}^{2}\bigg] \nonumber \\
    &\stackrel{\textcircled{\footnotesize 4}}{\leq} \frac{1}{2} \bigg[ 
    -||\nabla f(\bar{\boldsymbol{w}}^{k,j})||_{2}^{2} - || \sum_{u=0}^{U} \nabla f_{u}(\boldsymbol{w}_{u}^{k,j})||_{2}^{2} \nonumber \\
    & \hspace{12em} + \sum_{u=0}^{U} L^{2} || \bar{\boldsymbol{w}}^{k,j} - \boldsymbol{w}_{u}^{k,j}||_{2}^{2}\bigg], 
\end{align} 

where \textcircled{\footnotesize 1} is due to the fact that random variables $\xi_{u}^{k,j}$ and $\mathcal{U}$ are independent, \textcircled{\footnotesize 1} is because \textcircled{\footnotesize 2} $2\langle a,b \rangle = ||a||^{2} + ||b||^{2} - ||a-b||^{2}$, \textcircled{\footnotesize 3} holds due to the convexity of $||.||_{2}$, and \textcircled{\footnotesize 4} follows from Assumption \ref{Assumption1}.

\subsubsection{Proof of Lemma \ref{lemma2}} \label{prooflemma2}
We denote $k = i_{c}$ as the most recent global communication round, hence $\bar{\boldsymbol{w}}^{i_{c}+1} = \frac{1}{U} \sum_{u \in \mathcal{U}} \boldsymbol{w}_{u}^{i_{c + 1}}$. The local solution at device $u$ at any particular iteration $i > i_{c}$, where $i$ is assumed to represent the most recent iteration, encompassing all global and local iterations up to the current point, is written as:
$
    \boldsymbol{w}_{u}^{k,j} = \boldsymbol{w}_{u}^{i} = \boldsymbol{w}_{u}^{i-1} - \eta_{i_{c}} \Tilde{g}_{u}^{i-1} = \bar{\boldsymbol{w}}^{i_{c}+1} - \sum_{z=i_{c}+1}^{i-1} \eta_{i_{c}} \Tilde{g}_{u}^{z}
$
Next, we calculate the average virtual model at iteration $i$ as follows:
$
    \bar{\boldsymbol{w}}^{i} = \bar{\boldsymbol{w}}^{i_{c}+1} - \frac{1}{U} \sum_{u \in \mathcal{U}} \sum_{z=i_{c}+1}^{i-1} \eta_{i_{c}} \Tilde{g}_{u}^{z}.
$
Without loss of generality, assume that $i = s_{t} J + r$, where $s_{t}$ and $r$ represent the indices of global communication round and local updates, respectively. Now, consider that for $i_{c}+1 < i \leq i_{c} + T$, $\mathbb{E}_{i}||\bar{\boldsymbol{w}}^{i} - \boldsymbol{w}_{u}^{i}||$ is independent of time $i \leq i_{c}$ for $1 \leq u \leq U$. Consequently, for all iterations $1 \leq i \leq I$, where $I = KJ$, we can express,
\begin{align}
    & \frac{1}{KJ} \sum_{k=1}^{K} \sum_{j=1}^{J} \sum_{u=1}^{U} \mathbb{E}||\bar{\boldsymbol{w}}^{k,j} - \boldsymbol{w}_{u}^{k,j}||^{2} \nonumber \\
    &= \frac{1}{I} \sum_{s_{t}=1}^{\frac{I}{T} - 1} \sum_{r=1}^{T} \sum_{u=1}^{U} \mathbb{E}||\bar{\boldsymbol{w}}^{s_{t} E + r} - \boldsymbol{w}_{u}^{s_{t} E + r}||^{2}. 
\end{align}
We bound the term $\mathbb{E}||\bar{\boldsymbol{w}}^{i} - \boldsymbol{w}_{l}^{i}||^{2}$ for $i_{c}+1 \leq i = s_{t} J + r \leq i_{c} + J $ in three steps: (1) First, we connect this quantity to the variance between the stochastic gradient and the full gradient, (2) Then, we apply Assumption \ref{Assumption1} regarding unbiased estimation and i.i.d. mini-batch sampling, (3) We use Assumption \ref{Assumption3} to bound the final terms. In the following parts, we proceed to implement each of these steps. It is to note that $l$ is associated with individual device while $u$ is used for summing over devices.

\noindent
\textit{Relating to variance:}
\begin{align}
    &\mathbb{E}||\bar{\boldsymbol{w}}^{s_{t} E + r} - \boldsymbol{w}_{l}^{s_{t} E + r}||^{2} \nonumber \\
    &= \mathbb{E}|| \bar{\boldsymbol{w}}^{i_{c} + 1} - \bigg[ \sum_{z=i_{c}+1}^{i-1} \eta_{i_{c}} \Tilde{g}_{l}^{z} \bigg] - \bar{\boldsymbol{w}}^{i_{c}+1} \nonumber \\
    &\hspace{12em} + \bigg[ \frac{1}{U} \sum_{u \in \mathcal{U}} \sum_{z=i_{c}+1}^{i-1} \eta_{i_{c}} \Tilde{g}_{u}^{z} \bigg] ||^{2} \nonumber \\
    &\stackrel{\textcircled{\footnotesize 1}}{=} \mathbb{E}|| \sum_{z=1}^{r} \eta_{i_{c}} \Tilde{g}_{l}^{s_{t}+z} - \frac{1}{U} \sum_{u \in \mathcal{U}} \sum_{z=1}^{r} \eta_{i_{c}} \Tilde{g}_{u}^{s_{t}+z}||^{2} \nonumber \\
    &\stackrel{\textcircled{\footnotesize 2}}{=} 2 \mathbb{E}\bigg( \bigg[ ||\sum_{z=1}^{r} \eta_{i_{c}} \bigg[\Tilde{g}_{l}^{s_{t}J+z} - g_{l}^{s_{t}J+z}\bigg]||^{2} \nonumber \\
    &+ ||\sum_{z=1}^{r} \eta_{i_{c}} g_{l}^{s_{t}J+z}||^{2} \bigg] + ||\frac{1}{U} \sum_{u \in \mathcal{U}} \sum_{z=1}^{r} \eta_{i_{c}} \nonumber \\
    &\times \bigg[\Tilde{g}_{u}^{s_{t}J+z} - g_{u}^{s_{t}J+z}\bigg]||^{2} + ||\frac{1}{U} \sum_{u \in \mathcal{U}} \sum_{z=1}^{r} \eta_{i_{c}} g_{u}^{s_{t}J+z}||^{2}\bigg),
\end{align}
where \textcircled{\footnotesize 1} holds because $i = s_{t}J+r \leq i_{c} + J$ and \textcircled{\footnotesize 2} comes from Assumption \ref{Assumption1}.
\noindent
\textit{Unbiased estimation and i.i.d. sampling:}
\begin{align}
    &= 2 \mathbb{E}\bigg( \bigg[ \sum_{z=1}^{r} \eta_{i_{c}}^{2} ||\Tilde{g}_{l}^{s_{t}J+z} - g_{l}^{s_{t}J+z}||^{2} \nonumber \\
    &  \quad \quad + \sum_{p \neq q \vee l \neq v} \bigg\langle \eta_{i_{c}} \Tilde{g}_{l}^{p}-\eta_{i_{c}} g_{l}^{p}, \eta_{i_{c}} \Tilde{g}_{v}^{q}-\eta_{i_{c}} g_{v}^{q} \bigg\rangle \nonumber \\
    & \hspace{2em} + ||\sum_{z=1}^{r}\eta_{i_{c}} g_{l}^{s_{t}J+z}||^{2}\bigg] \nonumber \\
    &  \quad \quad + \frac{1}{U^{2}} \sum_{l \in \mathcal{U}} \sum_{z=1}^{r} \eta_{i_{c}}^{2} ||\Tilde{g}_{l}^{s_{t}J+z} - g_{l}^{s_{t}J+z}||^{2} \nonumber \\
    &  \quad \quad + \frac{1}{U^{2}} \sum_{p \neq q \vee l \neq v} \bigg\langle \eta_{i_{c}} \Tilde{g}_{l}^{p}-\eta_{i_{c}} g_{l}^{p}, \eta_{i_{c}} \Tilde{g}_{v}^{q}-\eta_{i_{c}} g_{v}^{q} \bigg\rangle \nonumber \\
    & \hspace{2em} + ||\frac{1}{U} \sum_{u \in \mathcal{U}} \sum_{z=1}^{r} \eta_{i_{c}} g_{u}^{s_{t}J+z} ||^{2} \bigg) \nonumber \\
\end{align}
\begin{align}   
    &= 2 \bigg( \bigg[ \sum_{z=1}^{r} \eta_{i_{c}}^{2} \mathbb{E} ||\Tilde{g}_{l}^{s_{t}J+z} - g_{l}^{s_{t}J+z}||^{2} \nonumber \\
    & \hspace{2em} + r \sum_{z=1}^{r}\eta_{i_{c}}^{2} \mathbb{E} ||g_{l}^{s_{t}J+z}||^{2}\bigg] \nonumber \\
    & \quad \quad + \frac{1}{U^{2}} \sum_{u \in \mathcal{U}} \sum_{z=1}^{r} \eta_{i_{c}}^{2} \mathbb{E} ||\Tilde{g}_{u}^{s_{t}J+z} - g_{u}^{s_{t}J+z}||^{2} \nonumber \\
    & \hspace{2em} + \frac{r}{U^{2}}  \sum_{u \in \mathcal{U}} \sum_{z=1}^{r} \eta_{i_{c}}^{2} \mathbb{E} ||g_{u}^{s_{t}J+z} ||^{2} \bigg). \label{eqn55}
\end{align}

\noindent
\textit{Using Assumption \ref{Assumption3}:}
Our next step is to bound the terms in \eqref{eqn55} using Assumption 3 as follows:
\begin{align}
    & \mathbb{E} ||\bar{\boldsymbol{w}}^{k,j} - \boldsymbol{w}_{l,k}^{t}||^{2} \leq 2 \bigg( \bigg[ \sum_{z=1}^{r} \eta_{i_{c}}^{2} \bigg[ C_{1} ||g_{l}^{s_{t}J+z}||^{2} + \frac{\sigma^{2}}{B} \bigg] \nonumber \\
    &+ r \sum_{z=1}^{r} \eta_{i_{c}}^{2} || g_{l}^{s_{t}J+z}||^{2} + \frac{1}{U^{2}} \sum_{u \in \mathcal{U}} \sum_{z=1}^{r} \eta_{i_{c}}^{2} \bigg[ C_{1} ||g_{u}^{s_{t}J+z}||^{2} \nonumber \\
    & \hspace{5em} + \frac{\sigma^{2}}{B} \bigg] + \frac{r}{U^{2}} \sum_{u \in \mathcal{U}} \sum_{z=1}^{r} \eta_{i_{c}}^{2} || g_{u}^{s_{t}J+z}||^{2} \bigg) \nonumber \\
    &= 2 \bigg( \bigg[ \sum_{z=1}^{r} \eta_{i_{c}}^{2} C_{1} ||g_{l}^{s_{t}J+z}||^{2} + \sum_{z=1}^{r} \eta_{i_{c}}^{2} \frac{\sigma^{2}}{B} \nonumber \\
    & \hspace{1em} + r \sum_{z=1}^{r} \eta_{i_{c}}^{2} || g_{l}^{s_{t}J+z}||^{2} \bigg] + \frac{1}{U^{2}} \sum_{u \in \mathcal{U}} \sum_{z=1}^{r} \eta_{i_{c}}^{2} C_{1} ||g_{u}^{s_{t}J+z}||^{2} \nonumber \\
    & \hspace{4em} + \sum_{z=1}^{r} \eta_{i_{c}}^{2} \frac{\sigma^{2}}{UB} + \frac{r}{U^{2}} \sum_{u \in \mathcal{U}} \sum_{z=1}^{r} \eta_{i_{c}}^{2} || g_{u}^{s_{t}J+z}||^{2} \bigg). \label{eqn56} \nonumber \\
\end{align}
Now we determine the upper bound for $\sum_{r=1}^{T} \sum_{u=1}^{U} [\mathbb{E} ||\bar{\boldsymbol{w}}^{k,j} - \boldsymbol{w}_{u}^{k,j}||]$ using \eqref{eqn56} as follows:
\begin{align*}
    &\sum_{r=1}^{T} \sum_{u=1}^{U} \bigg[\mathbb{E} ||\bar{\boldsymbol{w}}^{s_{t}J+z} - \boldsymbol{w}_{u}^{s_{t}J+z}||\bigg] \nonumber \\
\end{align*}
\begin{align}
    &\stackrel{\textcircled{\footnotesize 1}}{\leq} 2 \eta_{i_{c}}^{2} \bigg( \bigg[ \sum_{z=1}^{T}  C_{1} \sum_{l=1}^{U} ||g_{l}^{s_{t}J+z}||^{2} + \frac{J(J+1) \sigma^{2}}{2B} \nonumber \\
    &+ \frac{J(J+1)}{2} \sum_{z=1}^{T} \sum_{l=1}^{U} ||g_{l}^{s_{t}J+z}||^{2} + \frac{1}{U^{2}} \sum_{u \in \mathcal{U}} \sum_{z=1}^{T}  C_{1} \nonumber \\
    &\times ||g_{u}^{s_{t}J+z}||^{2} + \frac{J(J+1)\sigma^{2}}{2UB} \nonumber \\
    &+\frac{J(J+1)}{2U^{2}} \sum_{u \in \mathcal{U}} \sum_{z=1}^{T} || g_{u}^{s_{t}J+z}||^{2} \nonumber \\
    &= \frac{\eta_{i_{c}}^{2} (U+1)}{U} \bigg( \bigg[ (2 C_{1} + J(J+1)) \sum_{z=1}^{T} \sum_{u=1}^{U} || g_{u}^{s_{t}J+z}||^{2} \bigg] \nonumber \\
    & \hspace{12em} + \frac{J(J+1)\sigma^{2}}{B} \bigg), \label{eqn57}
\end{align}
where \textcircled{\footnotesize 1} comes from the fact that the terms $||g_{l}||^{2}$ are positive. Now, summing over global communication rounds in \eqref{eqn57} yields:
\begin{align}
     &\sum_{s_{t}=1}^{I/T-1} \sum_{r=1}^{T} \sum_{u=1}^{U} \bigg[\mathbb{E} ||\bar{\boldsymbol{w}}^{s_{t}J+z} - \boldsymbol{w}_{u}^{s_{t}J+z}||\bigg] \nonumber \\
     &\leq \frac{\eta_{i_{c}}^{2} (U+1)}{U} \bigg( \bigg[ (2 C_{1} \nonumber \\
     & \hspace{2em} + J(J+1)) \sum_{s_{t}=1}^{I/T-1} \sum_{z=1}^{T} \sum_{u=1}^{U} || g_{u}^{s_{t}J+z}||^{2} \bigg] \nonumber \\
     & \hspace{6em} + \frac{I(J+1)\sigma^{2}}{B} \bigg) \nonumber \\
     &= \frac{\eta_{i_{c}}^{2} (U+1)}{U} \bigg( \bigg[ (2 C_{1} + J(J+1)) \sum_{i=1}^{I} \sum_{u=1}^{U} || g_{u}^{i}||^{2} \bigg] \nonumber \\
     &\hspace{6em} + \frac{I(J+1)\sigma^{2}}{B} \bigg),
\end{align}
which leads to
\begin{align}
     &\frac{1}{I} \sum_{i=1}^{I} \sum_{u=1}^{U} \bigg[\mathbb{E} ||\bar{\boldsymbol{w}}^{i} - \boldsymbol{w}_{n}^{i}||\bigg] \nonumber \\
     &\stackrel{\textcircled{\footnotesize 1}}{\leq} \frac{(2 C_{1} + J(J+1))}{I} \frac{\lambda \eta_{i_{c}}^{2} (U+1)}{U} \sum_{i=0}^{I-1} || \sum_{u=1}^{U} g_{u}^{i}||^{2} \nonumber \\
     & \hspace{8em} + \frac{\eta_{i_{c}}^{2} I(U+1)(J+1)\sigma^{2}}{UB}, \label{eqn59}
\end{align}
where \textcircled{\footnotesize 1} follows from the definition of weighted gradient diversity and upper bound assumption in \eqref{eqn31}. Finally, \eqref{eqn59} can be written as:
\begin{align}
     &\frac{1}{KJ} \sum_{k=1}^{K} \sum_{j=1}^{J} \sum_{u=1}^{U} \bigg[\mathbb{E} ||\bar{\boldsymbol{w}}^{k,j} - \boldsymbol{w}_{u}^{k,j}||\bigg] \nonumber \\
     &\leq \frac{(2 C_{1} + J(J+1))}{KJ} \frac{\lambda \eta_{i_{c}}^{2} (U+1)}{U} \sum_{k=1}^{K} \sum_{j=1}^{J} || \sum_{u=1}^{U} g_{u}^{k,j}||^{2} \nonumber \\
     & \hspace{6em} + \frac{\eta_{i_{c}}^{2} KJ(U+1)(J+1)\sigma^{2}}{UB}. \label{eqn60}
\end{align}

\subsubsection{Proof of Lemma \ref{lemma3}} \label{prooflemma3}
We have
\begin{align}
    &\mathbb{E}\bigg[||\Tilde{g}^{k,j} - g^{k,j}||^{2}\bigg] \stackrel{\textcircled{\footnotesize 1}}{=} \mathbb{E}\bigg[||\frac{1}{U} \sum_{u=0}^{U} \Tilde{g}_{u}^{k,j} - g_{u}^{k,j}||^{2}\bigg] \nonumber \\
    &= \frac{1}{U^{2}} \mathbb{E}\bigg[\sum_{u=0}^{U} ||(\Tilde{g}_{u}^{k,j} - g_{u}^{k,j})||^{2}\bigg] \nonumber \\
    & \hspace{6em} + \sum_{i \neq u} \langle \Tilde{g}_{i,k}^{t} - g_{i,k}^{t}, \Tilde{g}_{u}^{k,j} - g_{u}^{k,j} \rangle \nonumber \\
    & \hspace{4em} + \frac{1}{U^{2}} \sum_{i \neq u} \langle  \mathbb{E} \bigg[\Tilde{g}_{i,k}^{t} - g_{i,k}^{t}\bigg], \mathbb{E} \bigg[\Tilde{g}_{u}^{k,j} - g_{u}^{k,j} \bigg] \rangle  \nonumber \\
    &\stackrel{\textcircled{\footnotesize 2}}{\leq} \frac{1}{U^{2}} \sum_{u=0}^{U} \bigg[ C_{1} ||g_{u}^{k,j}||^{2} + C_{2}^{2}\bigg] = \frac{C_{1}}{U^{2}} \sum_{u=0}^{U} ||g_{u}^{k,j}||^{2} + \frac{C_{2}^{2}}{U},
\end{align}
where we use the definition of $\Tilde{g}^{k,j}$ and $g^{k,j}$ in \textcircled{\footnotesize 1} and \textcircled{\footnotesize 2} directly follows from Assumption \ref{Assumption3}. It is important to note that Assumption \ref{Assumption3} implies $\mathbb{E}[\Tilde{g}_{u}^{k,j}] = g_{u}^{k,j}$. As a result, we obtain
\begin{align}
    \mathbb{E}\bigg[ ||\Tilde{g}^{k,j}||^{2} \bigg] &= \mathbb{E}\bigg[ ||\Tilde{g}^{k,j} - \mathbb{E} [ \Tilde{g}^{k,j} ]||^{2} \bigg] + ||\mathbb{E} [ \Tilde{g}^{k,j} ]||^{2} \nonumber \\
    &\stackrel{\textcircled{\footnotesize 1}}{\leq} \frac{C_{1}}{U^{2}} \sum_{u=0}^{U} ||g_{u}^{k,j}||^{2} + \frac{C_{2}^{2}}{U} + \frac{1}{U} \sum_{u=0}^{U} ||g_{u}^{k,j}||^{2} \nonumber \\
    &= \bigg(\frac{C_{1}+U}{U^{2}}\bigg) \sum_{u=0}^{U} ||g_{u}^{k,j}||^{2} + \frac{C_{2}^{2}}{U},
\end{align}
where \textcircled{\footnotesize 1} yields because $||\sum_{i=1}^{m} a_{i}||^{2} \leq m \sum_{i=1}^{m} ||a_{i}||^{2}$, with $a_{i} \in \mathbb{R}^{n}$. Using the upper bound over the weighted gradient diversity, $\lambda$,
\begin{align}
    \mathbb{E}\bigg[ ||\Tilde{g}^{k,j}||^{2}\bigg] \leq \lambda \bigg(\frac{C_{1}+U}{U^{2}}\bigg) ||\sum_{u=0}^{U} g_{u}^{k,j}||^{2} + \frac{C_{2}^{2}}{U},
\end{align}
results in the stated bound.

\subsection{Proof of Theorem 1} \label{prooftheorem1}
Using Lemma \ref{lemma1} and Lemma \ref{lemma2}, we continue to further upper bound \eqref{eqn35} as follows:
\begin{align}
    &\frac{1}{KJ} \sum_{k=1}^{K} \sum_{j=1}^{J} \mathbb{E}[f(\bar{\boldsymbol{w}}^{k,j+1}) - f(\bar{\boldsymbol{w}}^{k,j})] \nonumber \\
    &\leq \frac{1}{KJ} \sum_{k=1}^{K} \sum_{j=1}^{J} \Bigg(-\eta_{k} \mathbb{E}\bigg[\langle \nabla f(\bar{\boldsymbol{w}}^{k,j}), \Tilde{g}^{k,j} \rangle\bigg]\Bigg) \nonumber \\
    & \hspace{6em} + \frac{1}{KJ} \sum_{k=1}^{K} \sum_{j=1}^{J} \frac{\eta_{k}^{2}  L}{2} \mathbb{E}\bigg[||\Tilde{g}^{k,j}||^{2}\bigg] \nonumber \\
    &= \frac{1}{KJ} \sum_{k=1}^{K} \sum_{j=1}^{J} \bigg(-\frac{\eta_{k}}{2}||\nabla f(\bar{\boldsymbol{w}}^{k,j})||^2 \nonumber \\
    &- \frac{\eta_{k}}{2}||\sum_{u=1}^{U}\nabla f_{u}(\boldsymbol{w}_{u}^{k,j})||^{2}\bigg) \nonumber \\
    &+ \frac{\lambda \eta_{k} L^{2}}{2KJ} \frac{U+1}{U} \Bigg(\lambda \bigg[2C_{1}+J(J+1)\bigg] \eta_{k}^{2} \frac{1}{KJ} \sum_{k=1}^{K} \sum_{j=1}^{J}||^2 \nonumber \\
    &- \frac{\eta_{k}}{2}||\sum_{u=1}^{U}\nabla f_{u}(\boldsymbol{w}_{u}^{k,j})||^{2} \Bigg) \nonumber \\
    &+ \frac{KJ(L+1)\eta_{k}^{2} \sigma^{2}}{B} + \frac{1}{KJ} \sum_{k=1}^{K} \sum_{j=1}^{J} \frac{\lambda L \eta_{k}^{2}}{2} \lambda \bigg(\frac{C_{1}}{U}+1\bigg) \nonumber \\
    &\bigg[||\sum_{u=1}^{U}\nabla f_{u}(\boldsymbol{w}_{u}^{k,j})||^{2}\bigg] + \frac{L \eta_{k}^{2}}{2} \frac{\sigma^{2}}{UB}. \label{39} 
\end{align}

From \eqref{39}, we have
\begin{align}
    &\frac{1}{KJ} \sum_{k=1}^{K} \sum_{j=1}^{J} \mathbb{E}[f(\bar{\boldsymbol{w}}^{k,j+1}) - f(\bar{\boldsymbol{w}}^{k,j})] \nonumber \\
    &\stackrel{\textcircled{\footnotesize 1}}{\leq} - \frac{1}{KJ} \sum_{k=1}^{K} \sum_{j=1}^{J} \frac{\eta_{k}}{2}||\nabla f(\bar{\boldsymbol{w}}^{k,j})||^2 \nonumber \\
    &+ \frac{\eta_{k}^{3} L^{2} (J+1) \sigma^{2}}{B} \bigg(\frac{U+1}{U}\bigg) + \frac{L \eta_{k}^{2}}{2} \frac{\sigma^{2}}{UB}, \label{eqn40}
\end{align}
where \textcircled{\footnotesize 1} follows if the following condition holds:
\begin{align}
    &-\frac{\eta_{k}}{2} + \frac{\lambda (U+1) L^{2} \eta_{k}^{3} [2C_{1}+J(J+1)]}{2U} \nonumber \\
    & \hspace{6em} + \frac{\lambda L \eta_{k}^{2}}{2} \bigg(\frac{C_{1}}{U}+1\bigg) \leq 0.
\end{align}
In any kind of FL framework, setting the coefficient of the local gradients' sum to zero helps control variance from diverse client updates, ensuring stable convergence. This condition limits the influence of individual clients on the global model, preventing oscillations or divergence. It keeps updates bounded, promoting reliable convergence toward an optimal solution.
By rearranging \eqref{eqn40}, we get
\begin{align}
    \frac{1}{KJ} \sum_{k=1}^{K} \sum_{j=1}^{J} \mathbb{E}||\nabla f(\bar{\boldsymbol{w}}^{k,j})||^{2} &\leq \frac{2 [f(\bar{\boldsymbol{w}}^{1,0}) - f^{*}]}{\eta_{k} KJ} + \frac{L \eta \sigma^{2}}{UB} \nonumber \\
    &+ \frac{2 \eta_{k}^{2} \sigma^{2} L^{2} (J+1)}{B} \bigg(1+\frac{1}{U}\bigg). \label{eqn42new1}
\end{align}
The convergence upper bound presented in \eqref{eqn42new1} is dedicated to modality cluster $m$. We now bring notation $m$ back into our analysis in order to find the upper bound for the unified global model across all modalities:
\begin{align}
    \frac{1}{KJ} \sum_{k=1}^{K} \sum_{j=1}^{J} \mathbb{E}||\nabla f_{m}(\bar{\boldsymbol{w}}_{m}^{k,j})||^{2} &\leq \frac{2 [f_{m}(\bar{\boldsymbol{w}}_{m}^{1,0}) - f_{m}^{*}]}{\eta_{k} KJ} + \frac{L \eta \sigma^{2}}{UB} \nonumber \\
    &+ \frac{2 \eta_{k}^{2} \sigma^{2} L^{2} (J+1)}{B} \bigg(1+\frac{1}{U}\bigg). \label{eqn42new2}
\end{align}
Taking summation in both sides of \eqref{eqn42new2} over all the modality clusters, we get
\begin{align}
    &\frac{1}{KJ} \sum_{k=1}^{K} \sum_{j=1}^{J} \mathbb{E}||\sum_{m=1}^{M} \nabla f_{m}(\bar{\boldsymbol{w}}_{m}^{k,j})||^{2} \nonumber \\
    &\leq \frac{2 \sum_{m=1}^{M} [f_{m}(\bar{\boldsymbol{w}}_{m}^{1,0}) - f_{m}^{*}]}{\eta_{k} KJ} + \frac{M L \eta \sigma^{2}}{UB} \nonumber \\
    & \hspace{10em} + \frac{M 2 \eta_{k}^{2} \sigma^{2} L^{2} (J+1)}{B} \bigg(1+\frac{1}{U}\bigg). \label{eqn42new3}
\end{align}
Finally we arrive at
\begin{align}
    &\frac{1}{KJ} \sum_{k=1}^{K} \sum_{j=1}^{J} \mathbb{E}||\nabla f(\bar{\boldsymbol{w}_{g}})||^{2} \leq \frac{2 \sum_{m=1}^{M} [f_{m}(\bar{\boldsymbol{w}}_{m}^{1,0}) - f_{m}^{*}]}{\eta_{k} KJ} \nonumber \\
    & \hspace{10em} + \frac{M L \eta \sigma^{2}}{UB} \nonumber \\
    & \hspace{10em} + \frac{M 2 \eta_{k}^{2} \sigma^{2} L^{2} (J+1)}{B} \bigg(1+\frac{1}{U}\bigg). \label{eqn42new3}
\end{align}

In non-convex optimization, achieving a global minimum is often infeasible due to the landscape's complexity, filled with local minima and saddle points. Instead of focusing on bounding the distance between consecutive points, an alternative approach is to bound the squared norm of the gradient estimate. This approach helps gauge how close we are to a stationary point, where the gradient’s magnitude is minimal, indicating minimal change. By upper bounding the squared gradient, we can evaluate convergence towards a solution that may not be globally optimal, however is practically effective in reducing the loss.

\section{Latency Optimization for the Proposed Multimodal FL Framework}
In the FML framework, we have multiple clusters to consider. For simplicity, we focus on analyzing the round-trip latency for a specific data modality cluste $m$, without loss of generality, and therefore omit the notation $m$ in our discussion. The latency for UAV $u$ in communication round $k$ consists of five well-defined parts: data sensing, local model training, local embeddings and model uploading, server-side model training, and global model downloading. Each of these latency components is explicitly formulated as follows.  

\begin{enumerate}[label=\textit{\textbf{\arabic*}})]
    \item \textit{\textbf{Data Sensing Time:}} We assume that each UAV has a group of $C$ static targets within its range for sensing. The set of these $C$ targets is represented as $\mathcal{C} = \{1,2,\dots,C\}$. For radar sensing, the response of the target, denoted as $G_{c}$, is expressed as $G_{c} = g_{c} \hat{\beta} g_{c}$. Here, $\hat{\beta}$ is a constant dependent on the reflective properties of the target, $g_{c}$ represents the path loss which follows the free-space loss model \cite{dehkordi2022beam}, \cite{lu2021resource}, \cite{ji2020energy} and is given by $g_{c,u} = \frac{\hat{\alpha}}{\|q_{0}-q_{c,u}\|^{2}}, \quad \forall{c \in \mathcal{C}, u \in \mathcal{U}}$, where
 $q_{c,u}$ denotes the location of the $c^{\text{th}}$ target. Referring to \cite{chiriyath2015inner}, we formulate the radar estimation information rate as 
    \begin{align} \footnotesize
        R_{c,u}^{(k),\text{rad}} = \frac{\delta}{2 \mu} &\log_{2}(1 + \frac{2 \sigma_{\text{pre}}^{2} \hat{\gamma}^{2} B^{3} \mu G_{c,u} p_{\text{se},u}^{(k)}}{\sigma^{2}}),
    \end{align}
    where $\delta$ is the radar transmission duty ratio, $\mu$ denotes the radar pulse duration, $\hat{\gamma}$ represents a constant determined by the radar waveform shape, and $\sigma_{\text{pre}}^{2}$ indicates the variance of the predicted radar return. It is crucial that the radar estimation information rate is at least equal to a predefined threshold, denoted as  $\nu$, leading to
    \vspace{-5pt}
    \begin{equation}
        R_{c,u}^{(k),\text{rad}} \ge x_{c,u}^{(k)} \nu, \quad \forall{k \in \mathcal{K}, u \in \mathcal{U}},
    \end{equation}
    where $x_{c,u}^{(k)} \in \{0,1\}$ is the sensing scheduling. When $x_{c,u}^{(k)}=1$, it means that UAV $u$ chooses to sense target $c$ at global round $k$, while $x_{c,u}^{(k)}=0$ indicates that target $c$ is not sensed by UAV $u$. We assume that each target is selected and sensed by at most one UAV during each global round, i.e.,
    \begin{equation}
        \sum_{u=1}^{U} x_{c,u}^{(k)} \le 1, \quad \forall{c \in \mathcal{C}, k \in \mathcal{K}.}
    \end{equation}
    At each global round, each UAV performs radar sensing on its chosen static ground target located within its coverage area to capture radar echo signal reflected from it. This signal is converted into a set of data bits used for local model training. If UAV $u$ generates $D_{u}^{(k)}$ samples at communication round $k$, the data sensing time for UAV $u$ at round $k$ is 
    \begin{align}
        \mathbf{T}_{\text{se},u}^{(k)} = \frac{x_{c,u}^{(k)} D_{u}^{(k)}}{R_{c,u}^{(k),\text{rad}}},
    \end{align}
    where $R_{c,u}^{(k),\text{rad}}$ represents the radar measurement information rate at round $k$, indicating the amount of information UAV $u$ can extract from the radar measurements of the target $c$ per unit of time.
    The associated energy consumption of the UAV is
    \begin{align}
        E_{\text{se},u}^{(k)} = p_{\text{se},u}^{(k)} \mathbf{T}_{\text{se},u}^{(k)},
    \end{align}
    where $p_{\text{se}, u}^{(k)}$ is UAV $u$'s sensing transmit power at round $k$.

    \item \textit{\textbf{Local Model Training Time:}} The local computation time of UAV $u$ at round $k$ is calculated as
    \begin{align}
        \mathbf{T}_{\text{train},u}^{(k)} = \frac{J C_{u}^{(k)} D_{u}^{(k)}}{f_{u}^{(k)}},
    \end{align}
    where $C_{u}^{(k)}$ represents the number of CPU cycles required for UAV $u$ to process a sample during a local update, while $f_{u}^{(k)}$ denotes the CPU computation capability of UAV $u$ (cycles/s). The related UAV energy consumption is calculated as
    \begin{align}
        E_{\text{train},u}^{(k)} = J \zeta_{u}^{(k)} C_{u}^{(k)} D_{u}^{(k)} \left(f_{u}^{(k)}\right)^2,
    \end{align}
    where $\zeta_{u}^{(k)}$ represents the effective switching capacitance, which is influenced by the UAV's hardware and chip design \cite{6}.
    \item \textit{\textbf{Local Embeddings and Model Uploading Time:}} After local training, UAVs transmit the outputs of their encoders (embeddings), which are essential for decoder training at the server. Furthermore, UAVs must send their parameters for the global model aggregation. \textit{For the uploading of embeddings}, we assume that the amount each UAV needs to upload during every communication round at time slot $t$ remains fixed, denoted as $s_{e}[t]$. The time taken for UAV $u$ to upload local embeddings during round $k$ at time slot $t$ is expressed as
    \begin{align}
        \mathbf{T}_{\text{em-cm},u}^{(k)} = \frac{s_{e}[t]}{R_{u}^{(k)}[t]},
    \end{align}  
    where $R_{u}^{(k)}[t]$ represents the corresponding uplink transmission rate of UAV $u$ to the BS, which is written as \cite{2} \vspace{-12pt}
    \begin{align}
        R_{u}^{(k)}[t] &= B_{u} \log_{2}\left(1+\frac{g_{u, \text{BS}}^{(k)}[t] p_{\text{cm},u}^{(k)}[t]}{\sigma^{2}}\right) \notag \\
        &= B_{u} \log_{2}\left(1+\frac{\gamma_{0} p_{\text{cm},u}^{(k)}[t]}{\left(d_{u,\text{BS}}^{(k)}[t]\right)^2}\right), \label{eqn:Rn}
    \end{align}
    where $\gamma_{0} = \frac{\beta_{0}}{\sigma^{2}}$ is the reference signal-to-noise ratio (SNR). Moreover, $B_{u}$ is the communication bandwidth allocated for UAV $u$, $p_{\text{cm},u}^{(k)}[t]$ denotes the communication transmit power of UAV $u$ at round $k$ at time slot $t$, $\sigma^{2}$ is the additive white Gaussian noise (AWGN) power at the BS, and $d_{u,\text{BS}}^{(k)}[t]$ is the distance (LoS) between UAV $u$ and BS at round $k$ at time slot $t$ which is calculated by
    \begin{equation}
    d_{u,\text{BS}}^{(k)}[t] = \sqrt{(x_{u}^{(k)}[t])^{2}+(y_{u}^{(k)}[t])^{2}+H^{2}}.
    \end{equation}
    The energy consumption of the UAV over T time slots is \vspace{-8pt}
    \begin{align}
        E_{\text{em-cm},u}^{(k)} = \sum_{t=1}^{T}T_{\text{em-cm},u}^{(k)} p_{\text{cm},u}^{(k)}[t].
    \end{align}
    \textit{For model parameter uploading}, we assume that the model size is the same for all UAVs and that is $s_{l}[t]$. the time taken for UAV $u$ to upload the model parameters during round $k$ is expressed as
    \begin{align}
        \mathbf{T}_{\text{ml-cm},u}^{(k)} = \frac{s_{l}[t]}{R_{u}^{(k)}[t]},
    \end{align}  
    where $R_{u}^{(k)}[t]$ is the transmission rate (uplink) of UAV $u$ to BS at round $k$ at time slot $t$, which is written as \eqref{eqn:Rn}. 
    The UAV energy consumption over T time slots is \vspace{-8pt}
    \begin{align}
        E_{\text{ml-cm},u}^{(k)} = \sum_{t=1}^{T}\mathbf{T}_{\text{ml-cm},u}^{(k)} p_{\text{cm},u}^{(k)}[t].
    \end{align}
    Since embedding and model parameter aggregation at the BS is fast and efficient, we neglect the aggregation time in the system latency calculation.
    
    \item \textit{\textbf{Server-side Model Training:}} Once the unified embedding is received as input, the server trains its model to carry out specific tasks, such as classification. Let $J'$ denote the number of iterations for server-side training; the time required for server model training in round $k$ is computed as
    \begin{equation}
        \mathbf{T}_{\text{train},\text{BS}}^{(k)} = \frac{J' C_{\text{BS}}^{(k)} h^{(k)}}{f_{\text{BS}}^{(k)}},
    \end{equation}
    where $C_{\text{BS}}^{(k)}$ is the CPU cycles required per sample during server training, and $f_{\text{BS}}^{(k)}$ is the BS's CPU processing rate (cycles/s). 
    \item \textit{\textbf{Global Model Downloading Time:}} The BS sends the aggregated model parameters to the UAVs based on their modality for the next training round. The downlink rate from BS to UAV $u$ in round $k$ is given by \vspace{-10pt}
    \begin{align}
        R_{\text{BS}}^{(k)} &= B_{\text{BS}}\log_{2}\left(1+\frac{g_{\text{BS},u}^{(k)} p_{\text{cm},\text{BS}}^{(k)}}{\sigma^{2}}\right) \notag \\
        &= B_{\text{BS}} \log_{2}\left(1+\frac{\beta_{0} p_{\text{cm},\text{BS}}^{(k)}}{\sigma^{2} \left(d_{\text{BS},u}^{(k)}\right)^2}\right) \notag \\
        &= B_{\text{BS}} \log_{2}\left(1+\frac{\gamma_{0} p_{\text{cm},\text{BS}}^{(k)}}{\left(d_{u,\text{BS}}^{k}[t]\right)^2}\right),
    \end{align} 
    where $B_{\text{BS}}$ is the BS communication bandwidth, $g_{\text{BS},u}^{(k)}$ is the LoS channel gain between BS and UAV $u$ at round $k$ which is same as $g_{u,\text{BS}}^{(k)}[t]$, $p_{\text{cm},\text{BS}}^{(k)}$ represents the communication transmit power of BS at round $k$, $\sigma^{2}$ is the AWGN power at UAV $u$, $\gamma_{0} = \frac{\beta_{0}}{\sigma^{2}}$ denotes the reference signal-to-noise ratio (SNR), and $d_{\text{BS},u}^{(k)}$ is the LoS distance between BS and UAV $u$ at round $k$ which is same as $d_{u,\text{BS}}^{k}[t]$. Let the global model size be denoted as $s_{g}$, the time taken for UAV $u$ to download the global model during round $k$ is \vspace{-5pt}
    \begin{align} 
        \mathbf{T}_{\text{dl},u}^{(k)} = \frac{s_{g}}{R_{\text{BS}}^{(k)}}.
    \end{align}
\end{enumerate}

We assume that aggregation occurs only after the local models from all participating UAVs have reached the BS. Therefore, the total time for any communication round $k$ is determined by the UAV that takes the longest time, i.e., \vspace{-5pt}
\begin{align}
    \mathbf{T}^{(k)} = \underset{u \in \mathcal{U}}{\max}&\{\mathbf{T}_{\text{se},u}^{(k)}+\mathbf{T}_{\text{train},u}^{(k)}+\mathbf{T}_{\text{em-cm},u}^{(k)} \\
    &+\mathbf{T}_{\text{train},\text{BS}}^{(k)}+\mathbf{T}_{\text{ml-cm},u}^{(k)}+\mathbf{T}_{\text{dl},u}^{(k)}\}.
\end{align}
Therefore, the total FML latency is written as 
\begin{align}
    \mathbf{T}_{\text{total}}^{\text{FL}} = \sum_{k=1}^{K} \mathbf{T}^{(k)}.
\end{align}
As the BS typically has abundant energy resources and UAVs are energy-constrained, we focus on the energy consumption of the UAV, which is expressed as
\begin{align}
    E_{\text{total},u}^{\text{UAV}} = \sum_{k=1}^{K}\left(E_{\text{se},u}^{(k)}+E_{\text{train},u}^{(k)}+E_{\text{em-cm},u}^{(k)}+E_{\text{ml-cm},u}^{(k)}\right). \label{eqn:18}
\end{align}

\color{black}
\subsection{Problem Formulation}
This study focuses on reducing the latency of the UAV-FML system. Building on the above analysis, we define a system latency minimization problem that seeks to jointly optimize UAV trajectory ($x_{u}^{(k)}[t]$, $y_{u}^{(k)}[t]$), sensing scheduling $(x_{c,u}^{(k)})$, and resource allocation for both the UAV ($p_{\text{se},u}^{(k)}$, $p_{\text{cm},u}^{(k)}[t]$, $f_{u}^{(k)}$) and the BS ($p_{\text{cm},\text{BS}}^{(k)}$, $f_{\text{BS}}^{(k)}$).
\\ 
\textit{\textbf{Problem 1:}} \vspace{-10pt}
\begin{subequations} 
\begin{align}
\min_{\begin{aligned}
\end{aligned}} \quad & \mathbf{T}_{\text{total}}^{\text{FL}} \label{eqn:32a}\\ 
\textrm{s.t.} \quad & 0 \le p_{\text{se},u}^{(k)} \le P_{\text{se},u}^{\max}, \quad \forall{k,u} \label{eqn:32b}\\
& 0 \le p_{\text{cm},u}^{(k)}[t] \le P_{\text{cm},u}^{\max}, \quad \forall{k,u,t} \label{eqn:32c}\\
& 0 \le p_{\text{cm},\text{BS}}^{(k)} \le P_{\text{cm},\text{BS}}^{\max}, \quad \forall{k} \label{eqn:32d}\\
& 0 \le f_{u}^{(k)} \le f_{u}^{\max}, \quad \forall{k,u} \label{eqn:32e}\\
& 0 \le f_{\text{BS}}^{(k)} \le f_{\text{BS}}^{\text{max}}, \quad \forall{k} \label{eqn:32f}\\
& x_{c,u}^{(k)} \in \{0,1\}, \quad \forall{c,u,k} \label{eqn:32g0}\\
& \sum_{u=1}^{U} x_{c,u}^{(k)} \le 1, \quad \forall{c,k} \label{eqn:32g}\\
& R_{c,u}^{\text{rad}} \ge x_{c,u}^{(k)} \nu, \quad \forall{k,u} \label{eqn:32h}\\
& E_{\text{UAV},u}^{\text{total}} \le E_{u}^{\max}, \forall{u} \label{eqn:32i}\\
(x_{u}^{(k)}[t+1] - x_{u}^{(k)}[t])^2 &+ (y_{u}^{(k)}[t+1] - y_{u}^{(k)}[t])^2 \nonumber \\ 
& \quad \quad \quad \leq (V_{\max} \delta_{t})^{2}, \forall{k,u,t},\label{eqn:32j}
\end{align} 
\end{subequations} 
with \textbf{control variables:} $\{x_{u}^{(k)}[t]$, $y_{u}^{(k)}[t]$, $x_{c,u}^{(k)}$, $p_{\text{se},u}^{(k)}$, $p_{\text{cm},u}^{(k)}[t]$, $p_{\text{cm},\text{BS}}^{(k)}$, $f_{u}^{(k)}$, $f_{\text{BS}}^{(k)}\}$. Here,  $p_{\text{se},u}^{(k)} = \{p_{\text{se},1}^{(k)},p_{\text{se},2}^{(k)},\dots,p_{\text{se},U}^{(k)}\}$, $p_{\text{cm},u}^{(k)}[t] = \{p_{\text{cm},1}^{(k)}[t],p_{\text{cm},2}^{(k)}[t],\dots,p_{\text{cm},U}^{(k)}[t]\}$, and $f_{u}^{(k)}=\{f_{1}^{(k)},f_{2}^{(k)},\dots,f_{U}^{(k)}\}$. In this context, $P_{\text{se},u}^{\max}$, $P_{\text{cm},u}^{\max}$,  $f_{u}^{\max}$ represent maximum value of sensing transmit power, communication transmit power, and CPU processing rate of UAV $u$, respectively. Similarly, $P_{\text{cm},\text{BS}}^{\max}$ and $f_{\text{BS}}^{\max}$ refers to the highest level of communication transmit power and CPU processing rate of BS, respectively. $E_{u}^{\max}$ sets the value of maximum energy consumed by a UAV. Moreover, the power limits for sensing and communication transmission by UAVs are defined in \eqref{eqn:32b} and \eqref{eqn:32c}, respectively. The transmit power constraint for the BS is represented by  \eqref{eqn:32d}. The CPU processing rate for both UAVs and the BS is constrained in \eqref{eqn:32e} and \eqref{eqn:32f}, respectively. Furthermore, \eqref{eqn:32g0} and \eqref{eqn:32g} set the limits on the sensing scheduling, while \eqref{eqn:32h} ensures that the radar measurement information rate is above a specified threshold. \eqref{eqn:32i} governs the maximum energy consumption of UAV $u$, and \eqref{eqn:32j} restricts the maximum distance UAV $u$ can cover in a single time slot $t$.

\section{Proposed Solution for FML Latency Minimization}

\textbf{Problem 1} is difficult to solve in a straightforward manner and intractable for traditional convex solvers in its current form due to the non-convexity of the objective function and constraints. To address \textbf{Problem 1}, we break the originally formulated problem into three blocks or sub-problems (BCD technique). As a result, the control variables of \textbf{Problem 1} are partitioned in the following way: ($x_{c,u}^{(k)}, p_{\text{se},u}^{(k)}$), $(x_{u}^{(k)}[t], y_{u}^{(k)}[t], p_{\text{cm},u}^{(k)}[t], f_{u}^{(k)})$ and $(p_{\text{cm},\text{BS}}^{(k)}, f_{\text{BS}}^{(k)})$. Finally, we iteratively solve these three sub-problems until convergence.
\\
\textit{\textbf{Sub-problem 1 (Joint UAV sensing scheduling and power control): }} \vspace{-15pt}
\begin{subequations} 
\begin{align}
\min_{\begin{aligned}
    \text{\small $x_{c,u}^{(k)}$, $p_{\text{se},u}^{(k)}$}
\end{aligned}} \quad & \sum_{k=1}^{K}\max_{u \in \mathcal{U}} \left\{x_{c,u}^{(k)} D_{u}^{(k)}/R_{c,u}^{(k),\text{rad}}+\mathbf{T}_{\text{train},u}^{(k)}\right. \nonumber \\[-5pt]
& \left.+\mathbf{T}_{\text{em-cm},u}^{(k)}+\mathbf{T}_{\text{train},\text{BS}}^{(k)}+\mathbf{T}_{\text{ml-cm},u}^{(k)}+\mathbf{T}_{\text{dl},u}^{(k)}\right\} \label{eqn:33a}\\
\textrm{s.t.} \quad & 0 \le p_{\text{se},u}^{(k)} \le P_{\text{se},u}^{\max}, \quad \forall{k,u} \label{eqn:33b}\\
& x_{c,u}^{(k)} \in \{0,1\}, \quad \forall{c,k,u} \label{eqn:33c}\\
& \sum_{u=1}^{U} x_{c,u}^{(k)} \le 1, \quad \forall{c,k,u} \label{eqn:33d}\\
& R_{c,u}^{(k),\text{rad}} \ge x_{c,u}^{(k)} \nu, \quad \forall{k,u} \label{eqn:33e}\\
& E_{\text{UAV},u}^{\text{total}} \le E_{u}^{\max}, \forall{u}. \label{eqn:33f}
\end{align} 
\end{subequations}


\textit{\textbf{Sub-problem 2 (Joint UAV trajectory and resource allocation): }}
\vspace{-5pt}
\begin{subequations} 
\begin{align}
\min_{\begin{aligned}
    & \text{\small $x_{u}^{(k)}[t]$, $y_{u}^{(k)}[t]$,} \\[-0.5ex]
    & \text{\small $p_{\text{cm},u}^{(k)}[t]$, $f_{u}^{(k)}$}
\end{aligned}} \quad & \sum_{k=1}^{K}\max_{u \in \mathcal{U}} \left\{\mathbf{T}_{\text{se},u}^{(k)}+J C_{u}^{(k)} D_{u}^{(k)}/f_{u}^{(k)}\right. \nonumber \\[-15pt]
& \left.+s_{e}[t]/R_{u}^{(k)}[t]+\mathbf{T}_{\text{train},\text{BS}}^{(k)}\right. \nonumber \\
& \left.+s_{l}[t]/R_{u}^{(k)}[t]+\mathbf{T}_{\text{dl},u}^{(k)}\right\} \label{eqn:34a}\\ 
\textrm{s.t.} \quad & 0 \le p_{\text{cm},u}^{(k)}[t] \le P_{\text{cm},u}^{\max}, \forall{k,u,t} \label{eqn:34b}\\
& 0 \le f_{u}^{(k)} \le f_{u}^{\max}, \forall{k,u} \label{eqn:34c}\\
& E_{\text{UAV},u}^{\text{total}} \le E_{u}^{\max}, \forall{u} \label{eqn:34d}\\
(x_{u}^{(k)}[t+1] - x_{u}^{(k)}[t])^2 &+ (y_{u}^{(k)}[t+1] - y_{u}^{(k)}[t])^2 \nonumber\\
&\quad \quad \quad \leq (V_{\max} \delta_{t})^{2}, \forall{u,t}. \label{eqn:34e}
\end{align}
\end{subequations}
\vspace{-18pt}
\\
\textit{\textbf{Sub-problem 3 (BS resource allocation): }}
\vspace{-5pt}
\begin{subequations} 
\begin{align}
\min_{\begin{aligned}
    \text{\small $p_{\text{cm},\text{BS}}^{(k)}$, $f_{\text{BS}}^{(k)}$}
\end{aligned}} \quad & \sum_{k=1}^{K}\max_{u \in \mathcal{U}} \left\{\mathbf{T}_{\text{se},u}^{(k)}+\mathbf{T}_{\text{train},u}^{(k)}+\mathbf{T}_{\text{em-cm},u}^{(k)}\right. \nonumber\\[-5pt] 
& \left.+J' C_{\text{BS}}^{(k)} h^{(k)}/f_{\text{BS}}^{(k)}+T_{\text{ml-cm},u}^{(k)}+s_{g}/R_{\text{BS}}^{(k)}\right\}
\label{eqn:35a}\\ 
\textrm{s.t.} \quad & 0 \le p_{\text{cm},\text{BS}}^{(k)} \le P_{\text{cm},\text{BS}}^{\max}, \forall{k} \label{eqn:35b}\\
& 0 \le f_{\text{BS}}^{(k)} \le f_{\text{BS}}^{\text{max}}, \forall{k}. \label{eqn:35c}
\end{align}
\end{subequations} 
\noindent
\subsection{Sub-Problem 1: Optimizing UAV Sensing Scheduling and Power Control Given UAV Trajectory and Resource Allocation}
\textit{\textbf{Sub-problem 1}} is  non-convex due to the structure of the objective function in \eqref{eqn:33a}, the binary constraint in \eqref{eqn:33c}, and the constraints in \eqref{eqn:33d}, \eqref{eqn:33e}, and \eqref{eqn:33f}. Hence, we now focus on convexifying \eqref{eqn:33a}, \eqref{eqn:33c}, \eqref{eqn:33d}, \eqref{eqn:33e}, and \eqref{eqn:33f}. \textit{\textbf{For the objective function,}} we introduce a slack variable $\psi$ defined as:
\begin{align}
    \frac{x_{c,u}^{(k)} D_{u}^{(k)}}{\frac{\delta}{2 \mu} \log_{2}(1 + \frac{2 \sigma_{\text{pre}}^{2} \hat{\gamma}^{2} B^{3} \mu G_{c,u} p_{\text{se},u}^{(k)}}{\sigma^{2}})} \le \psi, \label{eqn:36}
\end{align}
Next, we introduce another slack variable $\iota$ and re-write \eqref{eqn:36} as:
\begin{subnumcases}{\eqref{eqn:36}\Leftrightarrow}
   & $x_{c,u}^{(k)} D_{u}^{(k)} \leq \psi \iota$, \label{eqn:37a'} \\
   & $\frac{\delta}{2 \mu} \log_{2}(1 + \frac{2 \sigma_{\text{pre}}^{2} \hat{\gamma}^{2} B^{3} \mu G_{c,u} p_{\text{se},u}^{(k)}}{\sigma^{2}}) \geq \iota$. \label{eqn:37b'}
\end{subnumcases}
It is clear that \eqref{eqn:36} can be rewritten as the system of inequalities above. Therefore, we will now proceed to examine the convexity of each inequality in this system.

\noindent
\uline{\textit{\textbf{\eqref{eqn:37a'}}}} We can re-write \eqref{eqn:37a'} equivalently as $\psi \iota \ge x_{c,u}^{(k)} D_{u}^{(k)}$, which is also represented as
\begin{align}
   &\psi \iota \geq x_{c,u}^{(k)} D_{u}^{(k)} \notag \\
   \Leftrightarrow \quad &\frac{1}{4} (\psi+\iota)^2 - \frac{1}{4} (\psi-\iota)^2 \geq x_{c,u}^{(k)} D_{u}^{(k)} \notag \\
   \Leftrightarrow \quad &\frac{1}{4} (\psi+\iota)^2 - x_{c,u}^{(k)} D_{u}^{(k)} \geq  \frac{1}{4} (\psi-\iota)^2. \label{eqn:38}
\end{align}
The right-hand side of \eqref{eqn:38} being already convex, we only need to convexify $(\psi+\iota)^2$. Using the first-order Taylor expansion, we approximate it as
\begin{align}
    (\psi+\iota)^2 \geq (\psi_{i}+\iota_{i})^2 + 2 (\psi_{i}+\iota_{i}) (\psi+\iota-\psi_{i}-\iota_{i}). \label{eqn:39'}
\end{align}
By putting \eqref{eqn:39'} into \eqref{eqn:38}, we arrive at
\begin{align}
    \frac{1}{4} &\left[(\psi_{i}+\iota_{i})^2 + 2 (\psi_{i}+\iota_{i}) (\psi+\iota-\psi_{i}-\iota_{i})\right] \notag \\
    &- x_{c,u}^{(k)} D_{u}^{(k)} \geq  \frac{1}{4} (\psi-\iota)^2. \label{eqn:40}
\end{align}
\uline{\textit{\textbf{\eqref{eqn:37b'}}}} For this, we incorporate the following inequality \cite{4}
\begin{align}
    \ln(1+z) \ge \ln(1+z_{i}) + \frac{z_{i}}{z_{i}+1} - \frac{(z_{i})^{2}}{z_{i}+1} \frac{1}{z}, \label{eqn:41}
\end{align}
to approximate the left-hand side of \eqref{eqn:37b'} as
\begin{align}
    \ln(1+\lambda_{i}) + \frac{\lambda_{i}}{\lambda_{i}+1} - \frac{(\lambda_{i})^{2}}{\lambda_{i}+1} \frac{1}{\lambda} \geq \frac{2 \mu \iota \ln{2}}{\delta}, \label{eqn:42'}
\end{align}
where $\lambda = \frac{2 \sigma_{\text{pre}}^{2} \hat{\gamma}^{2} B^{3} \mu G_{c,u} p_{\text{se},u}^{(k)}}{\sigma^{2}}$ and $\lambda_{i} = \frac{2 \sigma_{\text{pre}}^{2} \hat{\gamma}^{2} B^{3} \mu G_{c,u} \mathop{p_{\text{se},u}^{(k)}}_{i}}{\sigma^{2}}$.

\textit{\textbf{For the constraint \eqref{eqn:33c},}} in order to solve binary variable $x_{c,u}^{(k)}$, we first transform it into the continuous constraint, i.e., 
\begin{align}
    0 \le x_{c,u}^{(k)} \le 1. \label{eqn:33c'} 
\end{align}
\textit{\textbf{For the constraint \eqref{eqn:33e},}} we re-write it equivalently as
\begin{align}
    \frac{\delta}{2 \mu} \log_{2}(1 + \frac{2 \sigma_{\text{pre}}^{2} \hat{\gamma}^{2} B^{3} \mu G_{c,u} p_{\text{se},u}^{(k)}}{\sigma^{2}}) \ge x_{c,u}^{(k)} \nu, 
\end{align}
Similar to \eqref{eqn:42'}, we convexify it as
\begin{align}
    \ln(1+\lambda_{i}) + \frac{\lambda_{i}}{\lambda_{i}+1} - \frac{(\lambda_{i})^{2}}{\lambda_{i}+1} \frac{1}{\lambda} \geq \frac{2 \mu x_{c,u}^{(k)} \nu \ln{2}}{\delta}. \label{eqn:44}
\end{align}
\textit{\textbf{For the constraint \eqref{eqn:33f},}} we equivalently write it as
\begin{align}
    &E_{\text{UAV},u}^{\text{total}} \le E_{u}^{\max}, \forall{k}, \notag \\
    \Leftrightarrow \, &\sum_{k=1}^{K}
    \left(E_{\text{se},u}^{(k)}+E_{\text{train},u}^{(k)}+E_{\text{em-cm},u}^{(k)}+E_{\text{ml-cm},u}^{(k)}\right) \le E_{u}^{\max}, \forall{k}, \notag \\
    \Leftrightarrow \, &\sum_{k=1}^{K} \left(p_{\text{se},u}^{(k)} x_{c,u}^{(k)} D_{u}^{(k)}/R_{c,u}^{(k),\text{rad}} + E_{\text{train},u}^{(k)}+E_{\text{em-cm},u}^{(k)}\right. \notag \\
    &\quad \quad \quad \quad \quad \left. +E_{\text{ml-cm},u}^{(k)} \right) \le E_{u}^{\max}, \forall{k}. \label{eqn:45}
\end{align}
Apparently, \eqref{eqn:45} is non-convx because of the first term. By putting \eqref{eqn:36} into \eqref{eqn:45}, we get
\begin{align}
    &\sum_{k=1}^{K} \left(p_{\text{se},u}^{(k)} \psi + E_{\text{train},u}^{(k)}+E_{\text{em-cm},u}^{(k)}\right. \notag \\
    &\quad \quad \quad \quad \quad \left. +E_{\text{ml-cm},u}^{(k)} \right) \le E_{u}^{\max}, \forall{k}. \label{eqn:46}
\end{align}
For $p_{\text{se},u}^{(k)} > 0$ and $\psi > 0$, we utilize SCA to approximate $p_{\text{se},u}^{(k)} \psi$ as
\begin{align}
    p_{\text{se},u}^{(k)} \psi \leq \frac{1}{2} \frac{\psi_{i}}{\mathop{p_{\text{se},u}^{(k)}}_{i}} {p_{\text{se},u}^{(k)}}^2 + \frac{1}{2} \frac{\mathop{p_{\text{se},u}^{(k)}}_{i}}{\psi_{i}} \psi^2,
\end{align}
where $\mathop{p_{\text{se},u}^{(k)}}_{i}$ and $\psi_{i}$ represent the feasible values of $p_{\text{se},u}^{(k)}$ and $\psi$ at iteration $i$. Therefore, \eqref{eqn:46} (the equivalent of \eqref{eqn:33f}) is transformed into a convex form as
\begin{align}
    &\sum_{k=1}^{K} \left(\frac{1}{2} \frac{\psi_{i}}{\mathop{p_{\text{se},u}^{(k)}}_{i}} {p_{\text{se},u}^{(k)}}^2 + \frac{1}{2} \frac{\mathop{p_{\text{se},u}^{(k)}}_{i}}{\psi_{i}} \psi^2 + E_{\text{train},u}^{(k)}+E_{\text{em-cm},u}^{(k)}\right. \notag \\
    &\quad \quad \quad \quad \quad \left. +E_{\text{ml-cm},u}^{(k)} \right) \le E_{u}^{\max}, \forall{k}. \label{eqn:48}
\end{align}
After convexifying, \textit{\textbf{sub-problem 1}} is equivalently expressed as follows.\\
\noindent
\textit{\textbf{Sub-problem 1 (Equivalent):}}
\begin{subequations} 
\begin{align}
\min_{\begin{aligned}
    \text{\small $x_{c,u}^{(k)}$, $p_{\text{se},u}^{(k)}$}
\end{aligned}} \quad & \sum_{k=1}^{K}\max_{u \in \mathcal{U}} \left\{\psi+\mathbf{T}_{\text{train},u}^{(k)}\right. \nonumber \\[-5pt]
& \left.+\mathbf{T}_{\text{em-cm},u}^{(k)}+\mathbf{T}_{\text{train},\text{BS}}^{(k)}+\mathbf{T}_{\text{ml-cm},u}^{(k)}+\mathbf{T}_{\text{dl},u}^{(k)}\right\} \label{eqn:49a}\\
\textrm{s.t.} \quad & \eqref{eqn:40}, \eqref{eqn:42'}, \eqref{eqn:44}, \eqref{eqn:48}, \eqref{eqn:33c'}, \eqref{eqn:33b}, \eqref{eqn:33d}. \label{eqn:49b}
\end{align} 
\end{subequations}
Here, the sensing scheduling solution $x_{c,u}^{(k)}$ from \textbf{sub-problem 1} is continuous, and we approximate it to binary form before using it in the next sub-problems \cite{10}. If $x_{c,u}^{(k)} \geq 0.5$, it is rounded to 1; otherwise, it is set to 0.
\textit{\textbf{Regarding complexity}}, \textit{\textbf{Sub-problem 1 (Equivalent)}} involves $(2U)$ scalar decision variables and $(7U)$ convex constraints. According to the standard complexity results for interior-point methods \cite{ben2001lectures}, each iteration requires on the order of $\mathcal{O}\!\left((2U)^{2}\sqrt{7U}\right)$ operations.

\subsection{Sub-Problem 2: Optimizing  UAV Trajectory and Resource Allocation Given Sensing Scheduling and BS Resource Allocation} \textit{\textbf{Regarding the objective function,}} it is clear that the third and the fifth terms contribute to its non-convexity. To address this, we introduce a slack variable $g$ defined as:
\begin{align}
    \frac{s_{e}[t]+s_{l}[t]}{B_{u} \log_{2}\left(1+\frac{\gamma_{0} p_{\text{cm},u}^{(k)}[t]}{\left(d_{u,\text{BS}}^{(k)}[t]\right)^2}\right)} \leq g. \label{eqn:50}
\end{align}
Next, we define 3 additional slack variables $z$, $\gamma$, and $\alpha$, and re-write \eqref{eqn:50} in the following way:
\begin{subnumcases}{\eqref{eqn:50}\Leftrightarrow}
   & $s_{e}[t]+s_{l}[t] \leq g z$, \label{eqn:51a} \\
   & $B_{u} \log_{2}(1+\gamma) \geq z$, \label{eqn:51b} \\
   & $\frac{p_{\text{cm},u}^{(k)}[t]}{\alpha} \geq \gamma$, \label{eqn:51c} \\
   & $(x_{u}^{(k)}[t])^{2}+(y_{u}^{(k)}[t])^{2}+H^{2} \leq \alpha$. \label{eqn:51d}
\end{subnumcases}

\noindent
\eqref{eqn:50} is rewritten as this system of equations. We will now proceed to examine the convexity of each inequality in this system.\\
\noindent
\uline{\textit{\textbf{\eqref{eqn:51a}:}}} \eqref{eqn:51a} is equivalently re-written as $g z \geq s_{e}[t]+s_{l}[t]$ and further expressed as
\begin{align}
   &g z \geq s_{e}[t]+s_{l}[t] \notag \\
   \Leftrightarrow \quad &\frac{1}{4} (g+z)^2 - \frac{1}{4} (g-z)^2 \geq s_{e}[t]+s_{l}[t] \notag \\
   \Leftrightarrow \quad &\frac{1}{4} (g+z)^2 - s_{e}[t]-s_{l}[t] \geq  \frac{1}{4} (g-z)^2. \label{eqn:52}
\end{align}
The right-hand side of \eqref{eqn:52} is convex. As a result, we only approximate $(g+z)^2$. Employing the first-order Taylor expansion, we have
\begin{align}
    (g+z)^2 \geq (g_{i}+z_{i})^2 + 2 (g_{i}+z_{i}) (g+z-g_{i}-z_{i}). \label{eqn:53}
\end{align}
Replacing \eqref{eqn:53} into \eqref{eqn:52}, we now get
\begin{align}
    \frac{1}{4} &\left[(g_{i}+z_{i})^2 + 2 (g_{i}+z_{i}) (g+z-g_{i}-z_{i})\right] \notag \\
    & - s_{e}[t] - s_{l}[t] \geq  \frac{1}{4} (g-z)^2. \label{eqn:54}
\end{align}
\\
\uline{\textit{\textbf{\eqref{eqn:51b}:}}} By using the inequality in \eqref{eqn:41}, we approximate the left-hand side of \eqref{eqn:51b} as
\begin{align}
    \ln(1+\gamma_{i}) + \frac{\gamma_{i}}{\gamma_{i}+1} - \frac{(\gamma_{i})^{2}}{\gamma_{i}+1} \frac{1}{\gamma} \geq \frac{z \ln{2}}{B_{u}}. \label{eqn:55}
\end{align}
\\
\noindent
\uline{\textit{\textbf{\eqref{eqn:51c}:}}} We write \eqref{eqn:51c} in another way as
\begin{align}
    p_{\text{cm},u}^{(k)}[t] \geq \alpha \gamma. \label{eqn:56}
\end{align}
For $\alpha > 0$ and $\gamma > 0$, we use SCA to approximate right-hand side of \eqref{eqn:56} as
\begin{align}
    \alpha \gamma \leq \frac{1}{2} \frac{\gamma_{i}}{\alpha_{i}} \alpha^2 + \frac{1}{2} \frac{\alpha_{i}}{\gamma_{i}} \gamma^2, \label{eqn:57}
\end{align}
where $\alpha_{i}$ and $\gamma_{i}$ are the feasible values of $\alpha$ and $\gamma$ at iteration $i$. Thus, \eqref{eqn:56} (equivalent of \eqref{eqn:51c}) is turned into a convex form as
\begin{align}
    p_{\text{cm},u}^{(k)}[t] \geq \frac{1}{2} \frac{\gamma_{i}}{\alpha_{i}} \alpha^2 + \frac{1}{2} \frac{\alpha_{i}}{\gamma_{i}} \gamma^2. \label{eqn:58}
\end{align}
\\
\noindent
\uline{\textit{\textbf{\eqref{eqn:51d}:}}} \eqref{eqn:51d} is now convex and can be directly solved by convex solvers, such as CVX.

\textit{\textbf{For the constraint \eqref{eqn:34d},}} we re-write it as
\begin{align}
    &E_{\text{UAV},u}^{\text{total}} \le E_{u}^{\max}, \forall{n}, \notag \\
    \Leftrightarrow \, &\sum_{k=1}^{K}
    \left(E_{\text{se},u}^{(k)}+E_{\text{train},u}^{(k)}+E_{\text{em-cm},u}^{(k)}+E_{\text{ml-cm},u}^{(k)}\right) \le E_{u}^{\max}, \forall{n}, \notag \\
    \Leftrightarrow \, &\sum_{k=1}^{K} \Bigg(E_{\text{se},u}^{(k)} + J \zeta_{u}^{(k)} C_{u}^{(k)} D_{u}^{(k)} \left(f_{u}^{(k)}\right)^2 \notag \\
    & + \sum_{t=1}^{T} \Big(s_{e}[t]+s_{l}[t]\Big) p_{\text{cm},u}^{(k)}[t]/R_{u}^{(k)}[t] \Bigg) \le E_{u}^{\max}, \forall{n}. \label{eqn:59}
\end{align}
From \eqref{eqn:59}, the first and the second terms of the left-hand side are already in convex forms. However, the third term is non-convex. By Putting \eqref{eqn:50} into \eqref{eqn:59}, we reach at
\begin{align}
    \sum_{k=1}^{K} &\Bigg(p_{\text{se},u}^{(k)} \mathbf{T}_{\text{se},u}^{(k)} + J \zeta_{u}^{(k)} C_{u}^{(k)} D_{u}^{(k)} \left(f_{u}^{(k)}\right)^2 \notag \\
    &\quad  +\sum_{t=1}^{T} g p_{\text{cm},u}^{(k)}[t]\Bigg) \le E_{u}^{\max}, \forall{n}. \label{eqn:60}
\end{align}

Now, for $g > 0$ and $p_{\text{cm},u}^{(k)}[t] > 0$, we leverage SCA to transform $g p_{\text{cm},u}^{(k)}[t]$ into a convex form as
\begin{align}
    g p_{\text{cm},u}^{(k)}[t] \leq \frac{1}{2} \frac{\mathop{p_{\text{cm},u}^{(k)}[t]}_{i}}{g_{i}} g^2 + \frac{1}{2} \frac{g_{i}}{\mathop{p_{\text{cm},u}^{(k)}[t]}_{i}} {p_{\text{cm},u}^{(k)}[t]}^2, \label{eqn:61}
\end{align}
where $\mathop{p_{\text{cm},u}^{(k)}[t]}_{i}$ and $g_{i}$ denotes the feasible values of $p_{\text{cm},u}^{(k)}[t]$ and $g$ at iteration $i$, respectively. Thus, \eqref{eqn:59} (equivalent of \eqref{eqn:34d}) can be convexified as
\begin{align}
    &\sum_{k=1}^{K} \Bigg(E_{\text{se},u}^{(k)} + J \zeta_{u}^{(k)} C_{u}^{(k)} D_{u}^{(k)} \left(f_{u}^{(k)}\right)^2 \notag \\
    &  \sum_{t=1}^{T} \left(\frac{1}{2} \frac{\mathop{p_{\text{cm},u}^{(k)}[t]}_{i}}{g_{i}} g^2 + \frac{1}{2} \frac{g_{i}}{\mathop{p_{\text{cm},u}^{(k)}[t]}_{i}} {p_{\text{cm},u}^{(k)}[t]}^2\right)\Bigg) \le E_{u}^{\max}, \forall{n}. \label{eqn:62}
\end{align}
After going through the convexifying process, \textit{\textbf{sub-problem 2}} is written as follows.\\
\noindent
\textit{\textbf{Sub-problem 2 (Equivalent):}}
\begin{subequations} 
\begin{align}
\min_{\begin{aligned}
    & \text{\small $x_{u}^{(k)}[t]$, $y_{u}^{(k)}[t]$,} \\[-0.5ex]
    & \text{\small $p_{\text{cm},u}^{(k)}[t]$, $f_{u}^{(k)}$}
\end{aligned}} \quad & \sum_{k=1}^{K}\max_{u \in \mathcal{U}} \left\{\mathbf{T}_{\text{se},u}^{(k)}+J C_{u}^{(k)} D_{u}^{(k)}/f_{u}^{(k)}\right. \nonumber \\[-15pt]
& \left.+(s_{e}[t]+s_{l}[t])/R_{u}^{(k)}[t]+\mathbf{T}_{\text{train},\text{BS}}^{(k)}\right. \nonumber \\
& \left.+\mathbf{T}_{\text{dl},u}^{(k)}\right\} \label{eqn:63a}\\ 
\textrm{s.t.} \quad & \eqref{eqn:54}, \eqref{eqn:55}, \eqref{eqn:58}, \eqref{eqn:51d}, \eqref{eqn:62}, \nonumber \\
&\eqref{eqn:34b}-\eqref{eqn:34c}, \eqref{eqn:34e}. \label{eqn:63b}
\end{align}
\end{subequations}
\textit{\textbf{Regarding complexity}}, \textit{\textbf{Sub-problem 2 (Equivalent)}} is characterized by $(4U)$ scalar decision variables with $(8U)$ linear or quadratic constraints. Following the interior-point method analysis in \cite{ben2001lectures}, the resulting per-iteration computational complexity is expressed as $\mathcal{O}\!\left((4U)^{2}\sqrt{8U}\right)$.

\subsection{Sub-Problem 3: Optimizing BS Resource Allocation Given UAV Sensing, Trajectory and Resource Allocation Design} 
For given UAV sensing scheduling and resource allocation, we now focus on optimizing BS resource allocation. 
Similar to the other two sub-problems, \textit{\textbf{For the objective function,}} we define a slack variable $\Theta$ as: 
\begin{align}
    \frac{s_{g}}{B_{\text{BS}} \log_{2}\left(1+\frac{\gamma_{0} p_{\text{cm},\text{BS}}^{(k)}}{\left(d_{\text{BS},u}^{(k)}\right)^2}\right)} \leq \Theta \label{eqn:64}
\end{align}
As a result, \textit{\textbf{sub-problem 3}} is re-written as
\begin{subequations} 
\begin{align}
\min_{\begin{aligned}
    \text{\small $p_{\text{cm},\text{BS}}^{(k)}$, $f_{\text{BS}}^{(k)}$}
\end{aligned}} \quad & \sum_{k=1}^{K}\max_{u \in \mathcal{U}} \left\{\mathbf{T}_{\text{se},u}^{(k)}+\mathbf{T}_{\text{train},u}^{(k)}+\mathbf{T}_{\text{em-cm},u}^{(k)}\right. \nonumber\\[-5pt] 
& \left.+J' C_{\text{BS}}^{(k)} h^{(k)}/f_{\text{BS}}^{(k)}+\mathbf{T}_{\text{ml-cm},u}^{(k)}+\Theta\right\} \label{eqn:65a}\\ 
\textrm{s.t.} \quad & 0 \le p_{\text{cm},\text{BS}}^{(k)} \le P_{\text{cm},\text{BS}}^{\max}, \forall{k} \label{eqn:65b}\\
& 0 \le f_{\text{BS}}^{(k)} \le f_{\text{BS}}^{\text{max}}, \forall{k} \label{eqn:65c}\\
& \frac{s_{g}}{B_{\text{BS}} \Theta} \leq \log_{2}\left(1+\frac{\gamma_{0} p_{\text{cm},\text{BS}}^{(k)}}{\left(d_{\text{BS},u}^{(k)}\right)^2}\right), \forall{k}. \label{eqn:65d}
\end{align}
\end{subequations}
Here, \eqref{eqn:65a}, and \eqref{eqn:65b}-\eqref{eqn:65c} are already in convex form. As a result, we move forward to convexify \eqref{eqn:65d}.
\\
\uline{\textit{\textbf{\eqref{eqn:65d}:}}} We use the inequality in \eqref{eqn:41} to convexify right-hand side of \eqref{eqn:65d} as
\begin{align}
    \frac{s_{g} \ln2}{B_{\text{BS}} \Theta} \leq \ln\left(1+\xi_{i}\right)+\frac{\xi_{i}}{\xi_{i}+1} - \frac{{\xi_{i}}^2}{\xi_{i}+1}.\frac{1}{\xi} \label{eqn:66}
\end{align}
where $\xi = \frac{\gamma_{0} \mathop p_{\text{cm},\text{BS}}^{(k)}}{\left(d_{\text{BS},u}^{(k)}\right)^2}$ and $\xi_{i} = \frac{\gamma_{0} \mathop{p_{\text{cm},\text{BS}}^{(k)}}_{i}}{\left(d_{\text{BS},u}^{(k)}\right)^2}$.
After being convexified, \textit{\textbf{sub-problem 3}} is stated as follows.\\
\noindent
\textit{\textbf{Sub-problem 3 (Equivalent):}}
\begin{subequations} 
\begin{align}
\min_{\begin{aligned}
    \text{\small $p_{\text{cm},\text{BS}}^{(k)}$, $f_{\text{BS}}^{(k)}$}
\end{aligned}} \quad & \sum_{k=1}^{K}\max_{u \in \mathcal{U}} \left\{\mathbf{T}_{\text{se},u}^{(k)}+\mathbf{T}_{\text{train},u}^{(k)}+\mathbf{T}_{\text{em-cm},u}^{(k)}\right. \nonumber\\[-5pt] 
& \left.+J' C_{\text{BS}}^{(k)} h^{(k)}/f_{\text{BS}}^{(k)}+\mathbf{T}_{\text{ml-cm},u}^{(k)}+\Theta\right\} \label{eqn:67a}\\ 
\textrm{s.t.} \quad & \eqref{eqn:66}, \eqref{eqn:65b}-\eqref{eqn:65c}. \label{eqn:67b}
\end{align}
\end{subequations}
\textit{\textbf{Regarding complexity}}, \textit{\textbf{Sub-problem 3 (Equivalent)}} includes $(2K)$ scalar decision variables and $(2K+KU)$ linear or quadratic constraints. Based on the interior-point method framework in \cite{ben2001lectures}, the computational complexity is on the order of $\mathcal{O}\big((2K)^2 \sqrt{(2K+KU)}\big)$.

Building on our analysis, we are now prepared to solve the convex equivalents of the three sub-problems to derive the solutions to the original problem (\textbf{Problem 1}) using standard optimization techniques such as CVX. We solve the three blocks \textit{\textbf{(Sub-problem 1 (Equivalent), Sub-problem 2 (Equivalent), Sub-problem 3 (Equivalent))}} together to find the solutions for the original \textbf{Problem 1}, as outlined in Algorithm 1. Each SCA iteration, which sequentially solves the three convex sub-problems, takes approximately 13 seconds using YALMIP with MOSEK. The proposed algorithm converges within 5 SCA iterations in practice, resulting in an overall optimization time of around 65 seconds. Our BCD and SCA based iterative optimization procedure is terminated after a fixed number of iterations, which has been empirically found sufficient to ensure convergence and stable performance.
\makeatletter
\renewcommand{\Statex}{\item[]\hskip\ALG@thistlm}
\makeatother

\begin{algorithm}[t]
  \caption{SCA-based Joint Optimization Algorithm}
\begin{algorithmic}[1]
 \footnotesize
    \Statex \textbf{Input:} 
            \Statex Set the iteration index $i=0$;
            \Statex Define a feasible initial solution
            ${x_{c,u}^{(k)}}_0$, ${x_{u}^{(k)}[t]}_0$, ${y_{u}^{(k)}[t]}_0$, ${p_{\text{se},u}^{(k)}}_0$, ${p_{\text{cm},u}^{(k)}[t]}_0$, ${f_{u}^{(k)}}_0$, ${p_{\text{cm},\text{BS}}^{(k)}}_0$, ${f_{\text{BS}}^{(k)}}_0$ for Problem 1;
    \Statex \textbf{Repeat}
            \Statex Set $i \gets i+1$
            \Statex Solve \textbf{ Sub-problem 1 (Equivalent)} to obtain ${x_{c,u}^{(k)}}_i$, ${p_{\text{se},u}^{(k)}}_i$;
            \Statex Solve \textbf{ Sub-problem 2 (Equivalent)} to obtain ${x_{u}^{(k)}[t]}_i$, ${y_{u}^{(k)}[t]}_i$, ${p_{\text{cm},u}^{(k)}[t]}_i$, ${f_{u}^{(k)}}_i$;
            \Statex Solve \textbf{ Sub-problem 3 (Equivalent)} to obtain ${p_{\text{cm},\text{BS}}^{(k)}}_i$, ${f_{\text{BS}}^{(k)}}_i$;
    \Statex \textbf{Until} convergence. 
    \Statex \textbf{Output:}
            \Statex Optimal \textbf{${x_{c,u}^{(k)}}^*$}, \textbf{${x_{u}^{(k)}[t]}^*$}, \textbf{${y_{u}^{(k)}[t]}^*$}, \textbf{${p_{\text{se},u}^{(k)}}^*$}, \textbf{${p_{\text{cm},u}^{(k)}[t]}^*$}, \textbf{${f_{u}^{(k)}}^*$}, \textbf{${p_{\text{cm},\text{BS}}^{(k)}}^*$}, \textbf{${f_{\text{BS}}^{(k)}}^*$}.
\end{algorithmic} 
\end{algorithm}
\section{Simulation Results and Evaluation}

\begin{figure}[t!]
	\centering
	\begin{subfigure}[t]{0.25\textwidth}
		\centering
		\includegraphics[width=\linewidth]{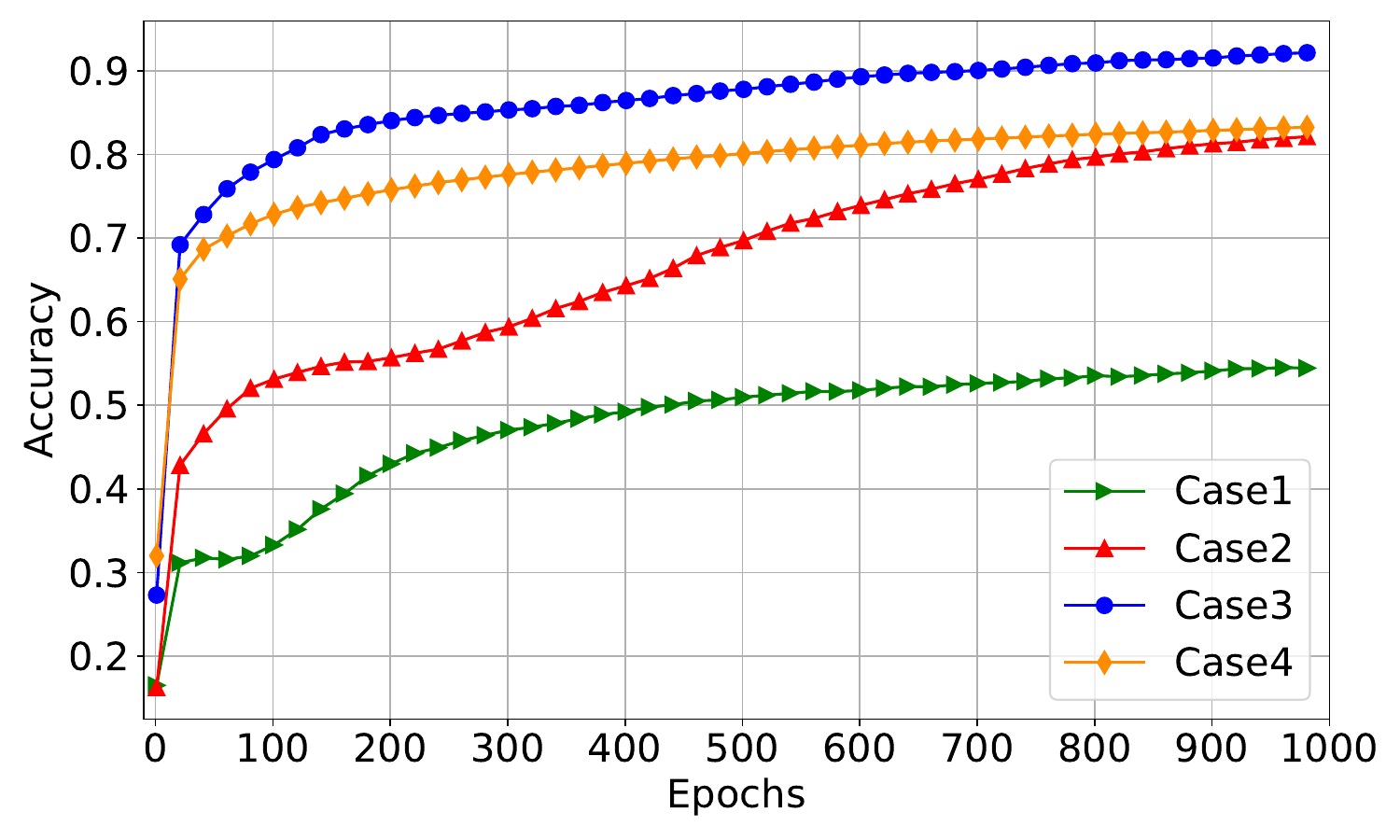} 
        \captionsetup{font=footnotesize}
		\caption{\footnotesize Evaluation of training performance with respect to IID accuracy. }
        \label{iid_acc}
	\end{subfigure}%
	~
	\begin{subfigure}[t]{0.25\textwidth}
		\centering
		\includegraphics[width=\linewidth]{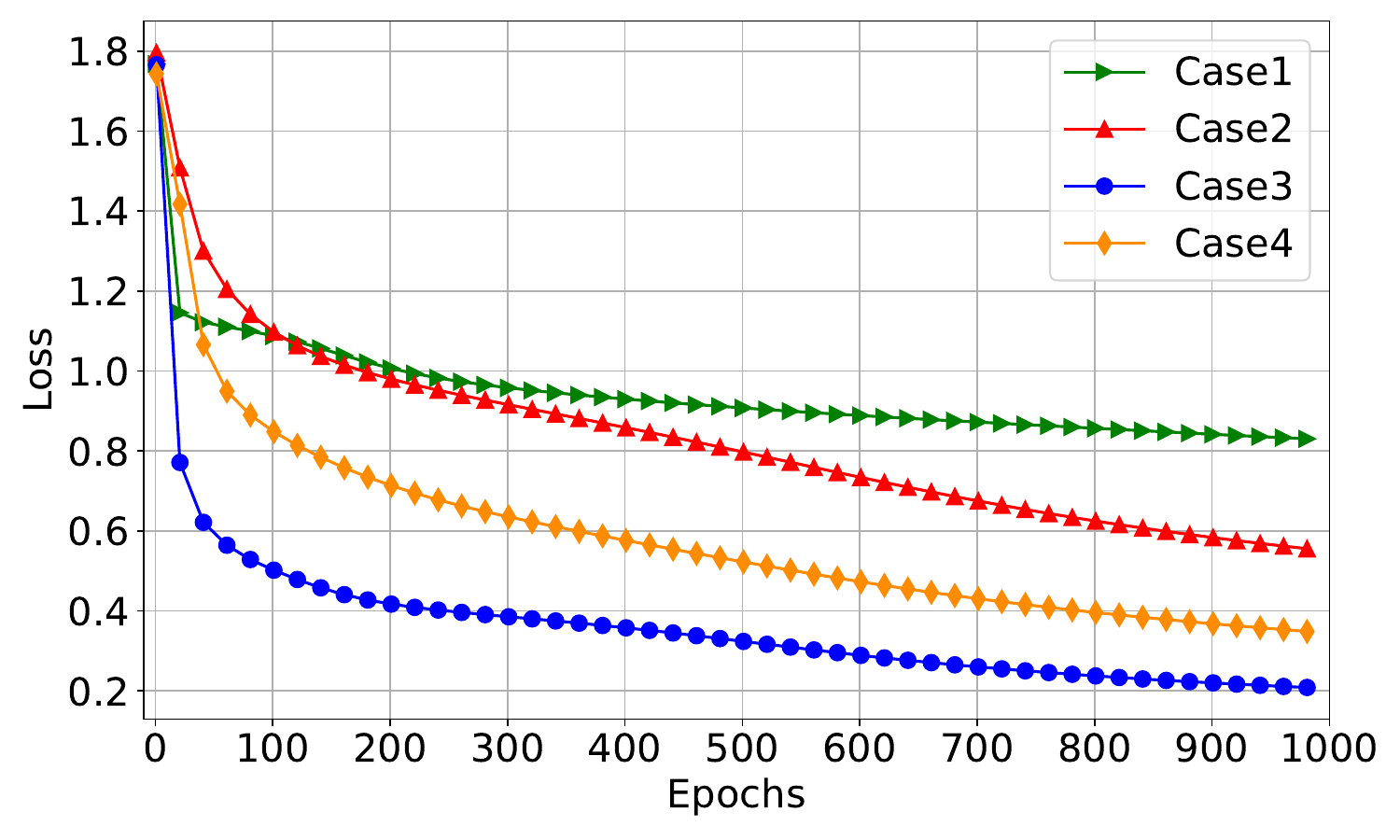} 
        \captionsetup{font=footnotesize}
		\caption{\footnotesize Evaluation of training performance with respect to IID loss.  }
        \label{iid_loss}
	\end{subfigure}
	\\ 
	\begin{subfigure}[t]{0.25\textwidth}
		\centering
		\includegraphics[width=\linewidth]{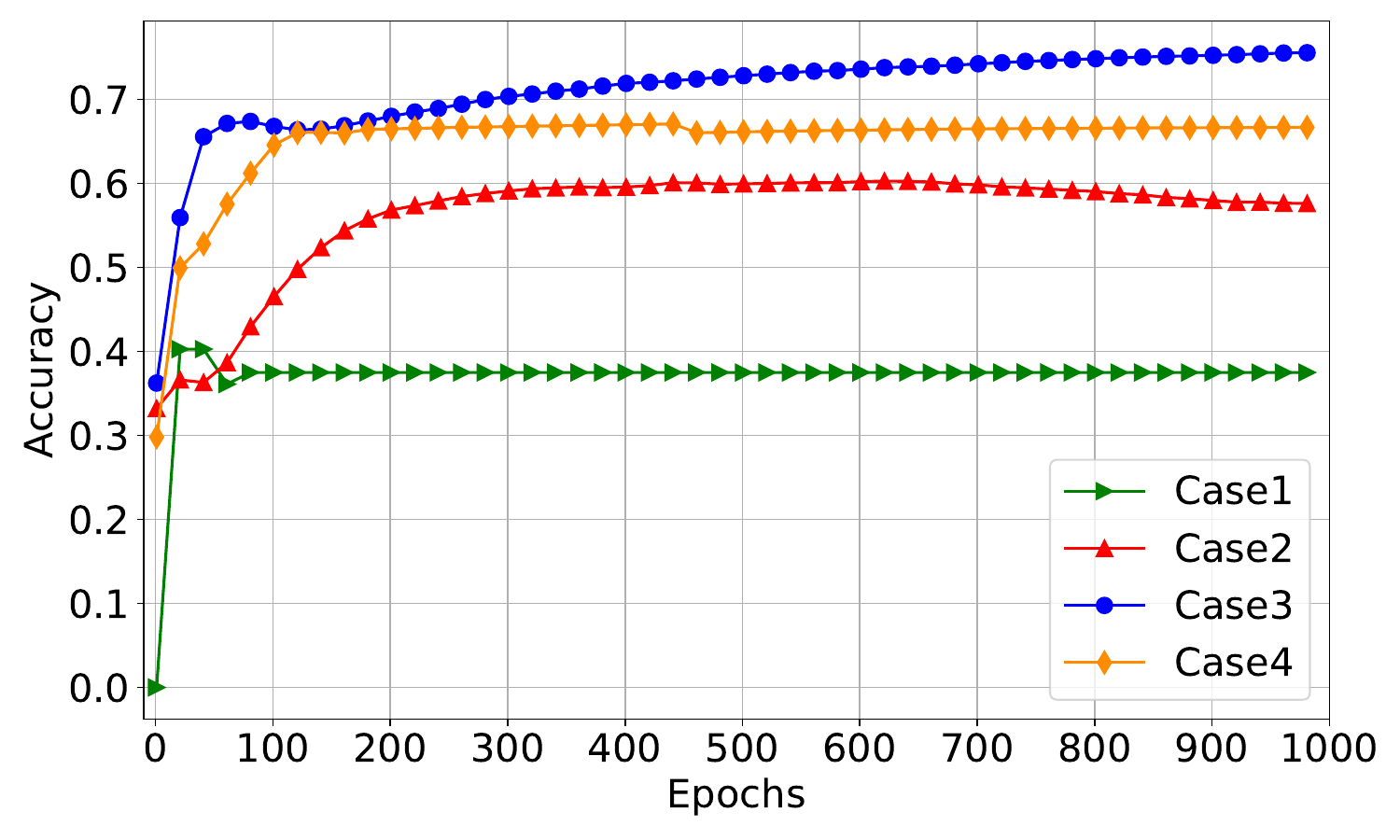} 
        \captionsetup{font=footnotesize}
		\caption{\footnotesize Evaluation of training performance with respect to non-IID accuracy. }
        \label{noniid_acc}
	\end{subfigure}%
	~
	\begin{subfigure}[t]{0.25\textwidth}
		\centering
		\includegraphics[width=\linewidth]{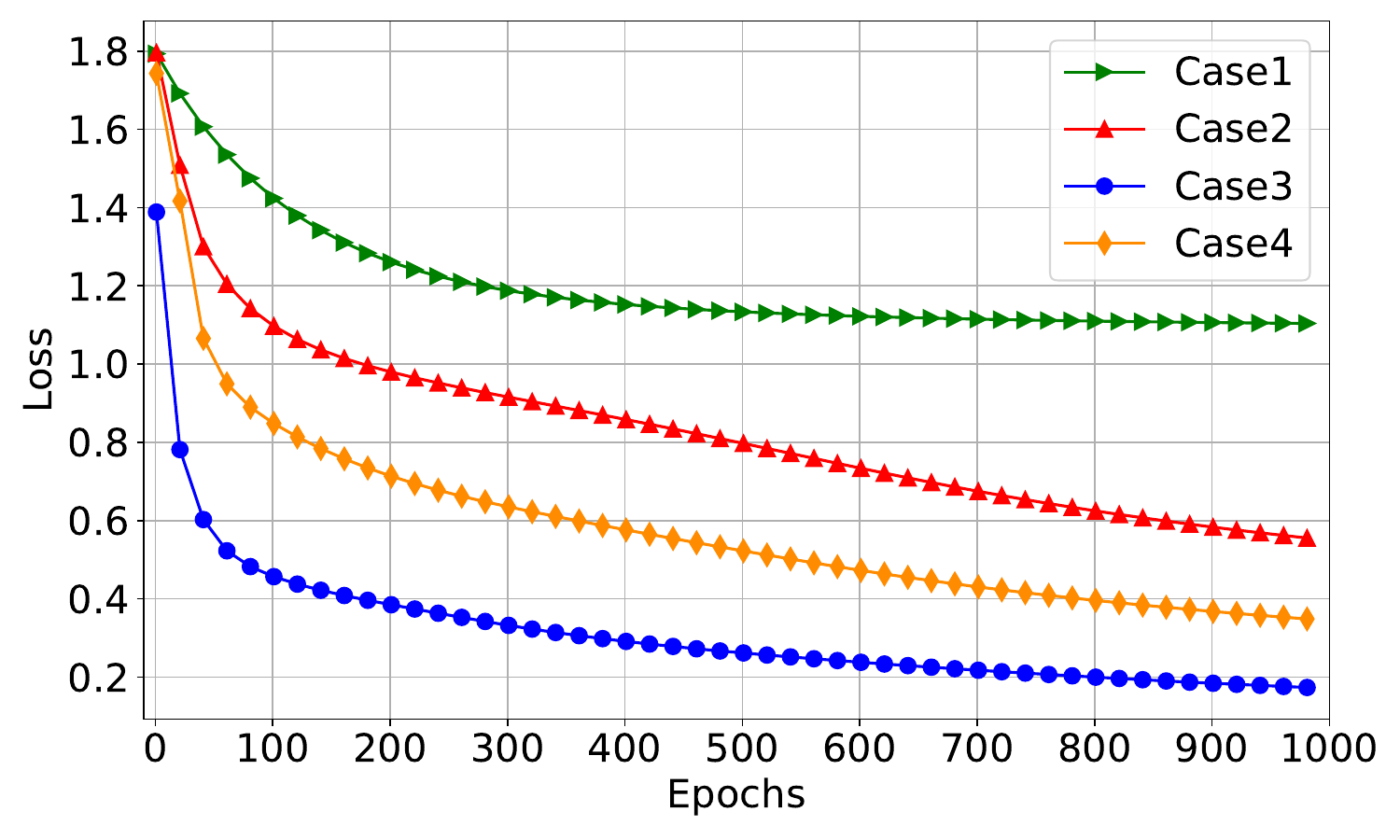} 
        \captionsetup{font=footnotesize}
		\caption{\footnotesize Evaluation of training performance with respect to non-IID loss.}
        \label{noniid_loss}
	\end{subfigure}
    \captionsetup{font=footnotesize}
	\caption{\footnotesize Comparison of various FML approaches on UCI HAR dataset.}
	\label{FL-compare-SVHN_Result}
	\vspace{-0.1in}
\end{figure}

\begin{figure}[ht!]
    \centering
    \includegraphics[width=0.99\linewidth]{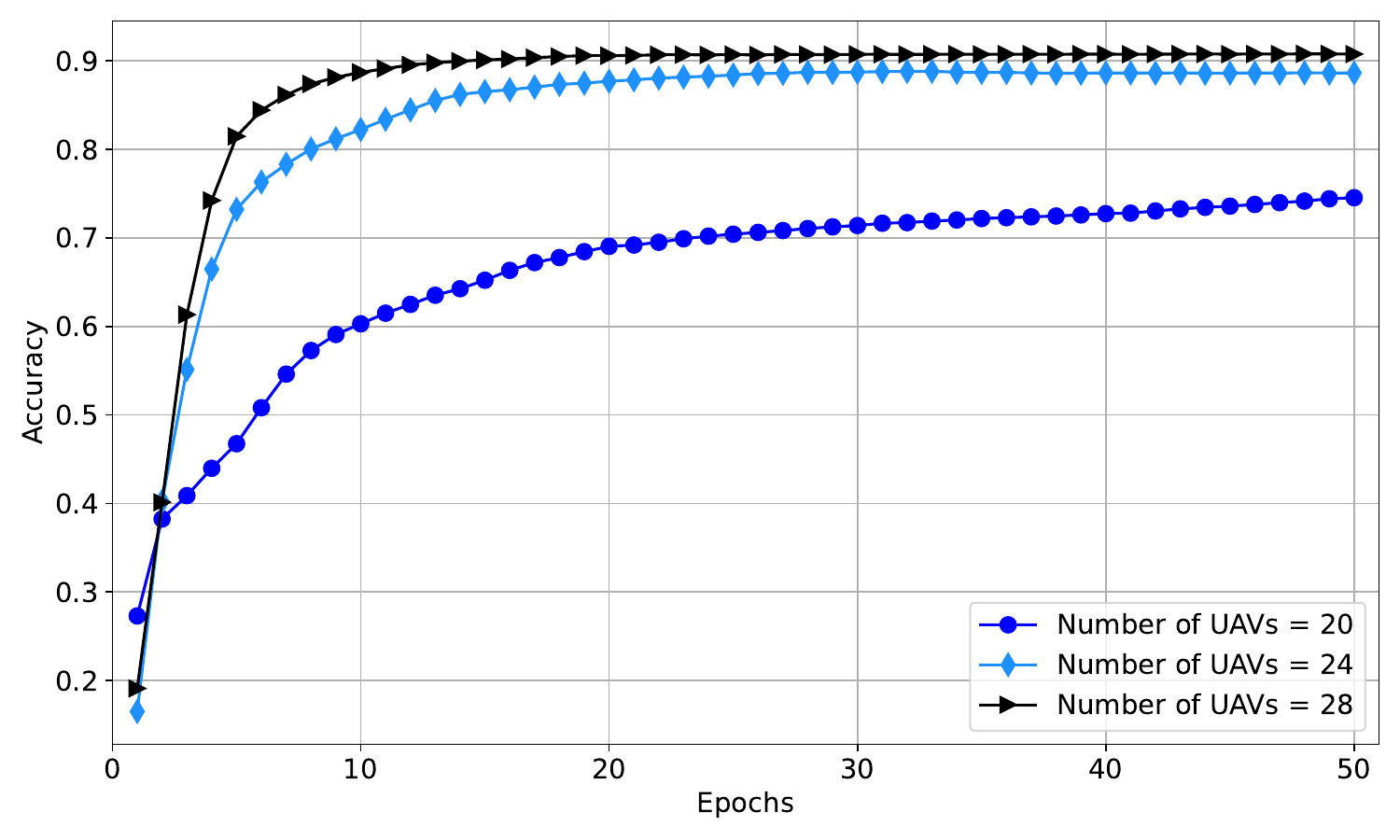}
    \captionsetup{font=footnotesize}
    \caption{\footnotesize Performance of our proposed FML scheme as the number of UAVs increases.}
    \label{fig:scalability}
\end{figure}
\noindent \textbf{1. Parameter Settings: }
\textit{For the FML model training simulation,} we utilize the \textit{UCI human activity recognition (HAR)} dataset \cite{anguita2013public}, which includes two modalities: data collected by gyroscope sensor in the first cluster and data collected by accelerometer sensor in the second cluster. The dataset consists of six activities: walking, going upstairs, going downstairs, standing, sitting, and lying down, collected from 30 subjects using a Samsung Galaxy S II. The data is split into $70\%$ for training and $30\%$ for testing.

In our simulation setup, we involve 20 UAVs, where 10 UAVs are responsible for sensing and training on accelerometer data, while the remaining 10 focus on gyroscope data. The experiments are conducted on a server equipped with an Intel Core™ i7-8700 CPU and 16 GB of RAM, running TensorFlow version 2.12.1. Our proposed algorithm is evaluated under different data configurations in both IID and non-IID scenarios. During our simulation, we measure the round latency using Python's time module. The IID setting creates homogeneous learning conditions across the network. It is because in the IID scenario, the feature distribution is uniform across all the UAVs, which ensures that each UAV receives a balanced mix of data from all activities. On the other hand, the non-IID setting introduces significant variability in feature distribution, with each UAV focusing on specific activities. This tests the model's robustness in learning from diverse and unbalanced data distributions. For the simulations, stochastic gradient descent (SGD) optimizer is used for both encoder and decoder updates. A learning rate of $0.01$ is used. Our proposed FML scheme (\textit{Case3}), incorporating two modality clusters (gyroscope and accelerometer), is compared against two other unimodal baselines: \textit{Case1} (all UAVs using gyroscope data) and \textit{Case2} (all UAVs using accelerometer data). In addition, we include a new baseline, Case4, where all UAVs use accelerometer data with the FedProx algorithm, to further benchmark our approach against a state-of-the-art unimodal FL variant.

\begin{figure}[t!]
    \centering
    \includegraphics[width=0.99\linewidth]{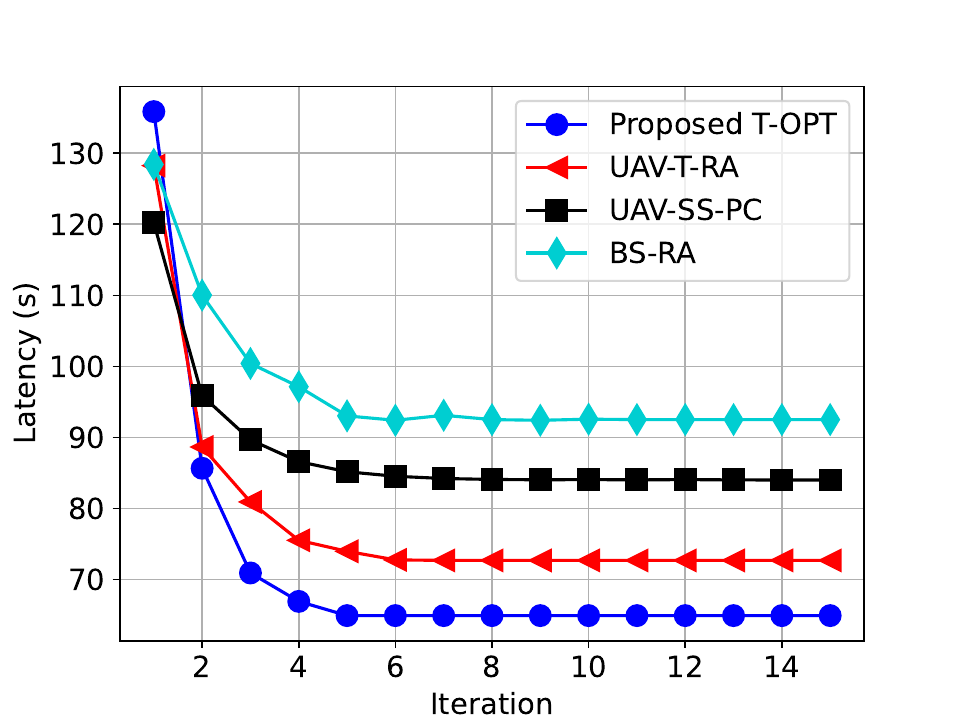}
    \captionsetup{font=footnotesize}
    \caption{\footnotesize Latency comparison.}
    \label{fig:sub11}
\end{figure}

\begin{figure*}[t!]
    \centering
    \footnotesize
    \begin{subfigure}[t]{0.3\linewidth}
        \centering
        \includegraphics[width=\linewidth]{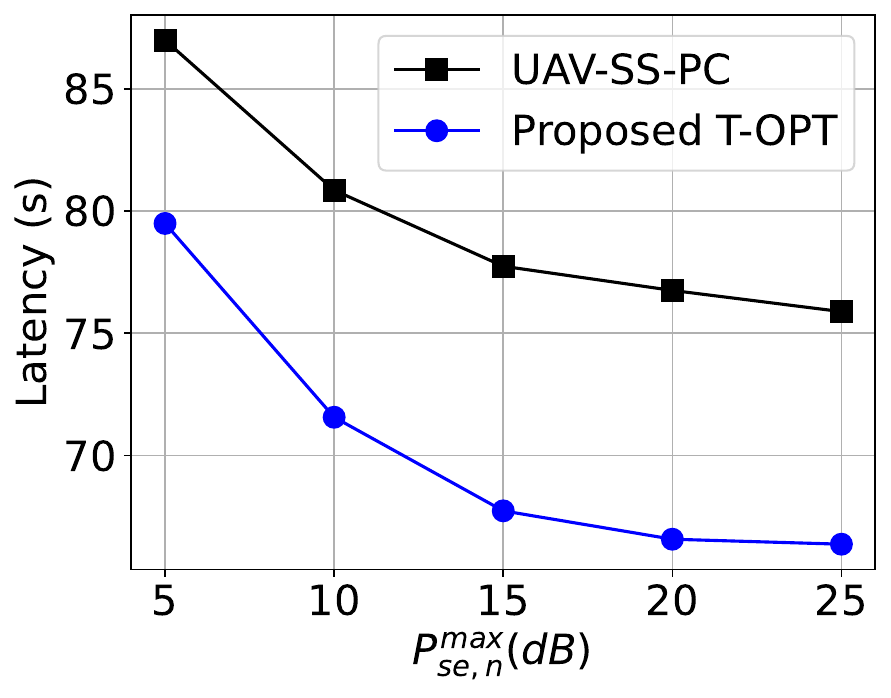}
        \captionsetup{font=footnotesize}
        \caption{\footnotesize System latency against UAVs' maximum sensing power.}
        \label{lat_comp_sub1}
    \end{subfigure}
    ~
    \begin{subfigure}[t]{0.3\linewidth}
        \centering
        \includegraphics[width=\linewidth]{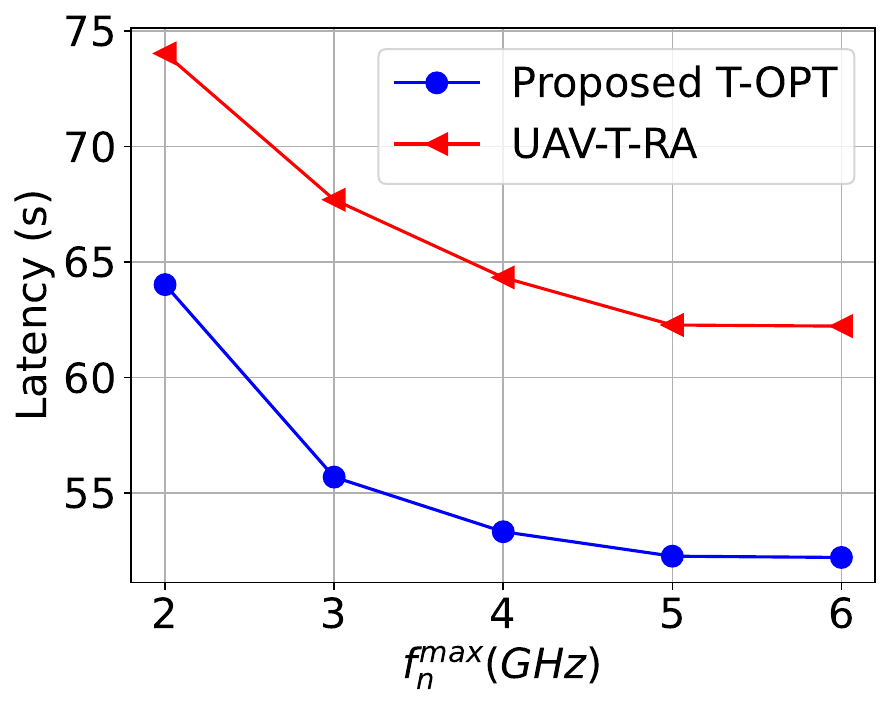}
        \captionsetup{font=footnotesize}
        \caption{\footnotesize System latency against UAVs' maximum frequency.}
        \label{lat_comp_sub2}
    \end{subfigure}
    ~
    \begin{subfigure}[t]{0.3\linewidth}
        \centering
        \includegraphics[width=\linewidth]{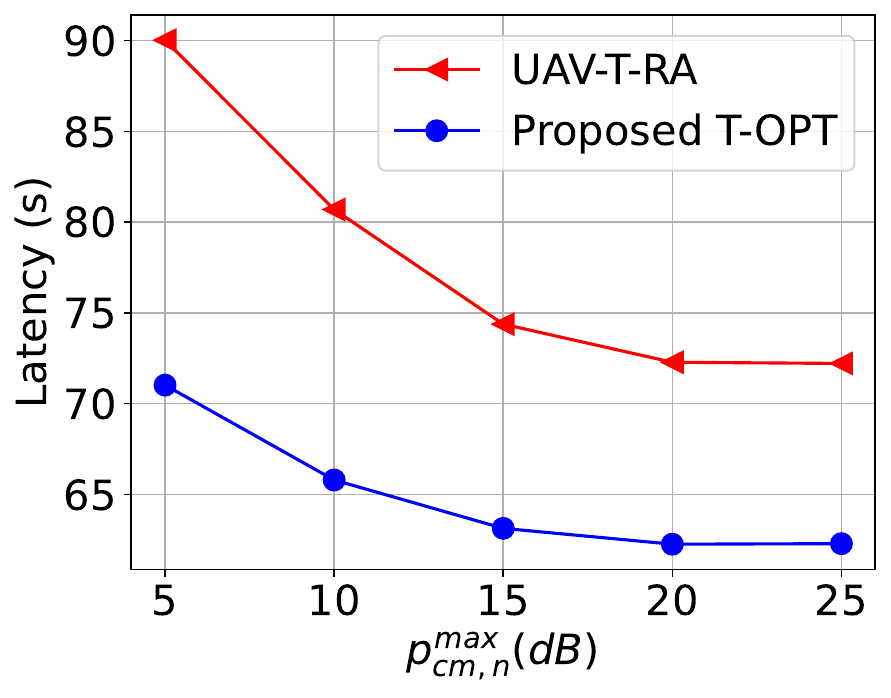}
        \captionsetup{font=footnotesize}
        \caption{\footnotesize System latency against UAVs' maximum communication transmit power.}
        \label{lat_comp_sub2_bar}
    \end{subfigure}
    \captionsetup{font=footnotesize}
    \caption{Comparison of system latency with baseline schemes (UAV-SS-PC and UAV-T-RA).}
\end{figure*}

\textit{For the FML latency simulation}, we verify the performance of our proposed joint optimization of UAV’s sensing scheduling, power control, trajectory, and resource allocation as well as resource allocation of BS (denoted as \textit{T-OPT}). We consider practical scenarios while setting parameter values. All simulations were performed in MATLAB using the YALMIP toolbox and the MOSEK solver. The system bandwidth is set to 20 MHz \cite{6}, with the maximum sensing transmit power of UAV, $P_{\text{se},u}^{\text{max}}$, ranging from 5 to 25 dB. The maximum communication transmit power of UAV, $P_{\text{cm},u}^{\text{max}}$, and of BS, $P_{\text{cm},\text{BS}}^{\text{max}}$,  are configured within the ranges of [5-25] dB and [15-35] dB, respectively. The maximum CPU cycle frequency for UAV is set to $f_{u}^{\text{max}}$ = 2 GHz, while for the BS, $f_{\text{BS}}^{\text{max}}$ = 10 GHz \cite{6}. The noise variance is considered to be $\sigma^2$ = -80 dBm \cite{7}. The effective switched capacitance for UAV's local computation is $\zeta_{u}^{(k)} = 10^{-28}$ \cite{6}. Each UAV performs a total of $J$ = 15 local iterations.

To compare, we evaluate the performance of our proposed joint optimization scheme (\textit{T-OPT}) alongside the following three benchmark schemes that do not employ joint optimization of trajectory and/or resource allocation: (1) sensing scheduling and power control of UAV which is written as \textit{UAV-SS-PC}), (2) trajectory and resource allocation of UAV which is denoted as \textit{UAV-T-RA}), and (3) resource allocation of BS which is represented as \textit{BS-RA}).

\vspace{2mm}
\noindent \textbf{2. FML Model Training Performance:} As mentioned before, we use the \textit{UCI HAR Dataset} \cite{anguita2013public} to simulate FL model training for a task that involves human activity recognition. Fig.~\ref{iid_acc} shows the accuracy against the number of epochs and compares our proposed FML scheme with Case1 and Case2 in terms of IID accuracy. Our scheme achieves 67.99\% higher accuracy than Case1, 11.98\% higher accuracy than Case2, and $10.87\%$ higher accuracy than Case4, showing the benefit of multimodal learning even compared with a FedProx-enhanced unimodal baseline. By incorporating multiple data modalities, FML captures a comprehensive view of underlying phenomena, creating more accurate models. Hence, this multimodal approach outpeforms unimodal schemes, reducing reliance on a single data modality. When UAVs are equipped with multiple sensors, incorporating data from all sensors leads to more precise and dependable predictions. This approach also improves the model's robustness against noise, data gaps, or irregularities. The reason is that the varying features from each data modality work together, strengthening the model's capacity to identify intricate patterns and correlations.

Fig.~\ref{iid_loss} illustrates the loss versus the number of epochs and assesses our proposed FML approach with other schemes in terms of IID loss. The graph clearly shows that the trend in accuracy is consistent in terms of loss, with  our proposed scheme achieving 75.13\% lower loss compared to case1, 62.54\% lower loss compared to case2, and 42.86\% lower loss compared to case4. FML harnesses the synergy of sensors to improve prediction accuracy and minimizes loss by utilizing feature diversity \cite{peng2024fedmm}.



Fig.~\ref{noniid_acc} compares our proposed scheme with case1 and case2 based on non-IID accuracy, showing 101.68\%, 31.61\%, and 13.43\% higher accuracy than case1, case2, and case4, respectively. The non-IID loss performance of our proposed scheme in Fig.~\ref{noniid_loss} outperforms other cases as well.

Fig.~\ref{fig:scalability} illustrates the training accuracy of our proposed scheme for different numbers of participating UAVs. As the number of UAVs increases from 20 to 24, there is a substantial improvement in both convergence speed and final accuracy, highlighting the impact of more diverse and representative data in the federated training process. Increasing from 24 to 28 UAVs yields only a slight improvement, suggesting that adding more UAVs beyond a certain point offers limited additional benefit. However, the model converges faster, indicating that additional UAVs still help accelerate training even if the accuracy gain is limited.

\noindent \textbf{3. FML Model Training Latency Performance: }
Fig.~\ref{fig:sub11} evaluates the performance of our proposed SCA- and BCD-based convex optimization algorithm in terms of the system latency (second) against the number of iterations. Compared to other schemes, our approach achieves much lower system latency for the FL system, showing the merit of our joint optimization design. Numerically, our proposed scheme maintains a steady latency after the fifth iteration, resulting in $29.39\%$, $11.96\%$, and $42.49\%$ lower latency than UAV-SS-PC, UAV-T-RA, and BS-RA schemes, respectively.


We also analyze the latency performance of our proposed method across various scenarios. Fig.~\ref{lat_comp_sub1} shows the latency (second) against the maximum UAV sensing power. This figure compares our proposed joint optimization scheme T-OPT with scheme UAV-SS-PC. As the maximum UAV sensing power increases, latency decreases for both schemes. However, our scheme T-OPT outperforms UAV-SS-PC, achieving $14.33\%$ lower latency. Increased sensing power accelerates the data collection process and improves the quality of the data. This, in turn, leads to quicker data processing and transmission, ultimately lowering the overall system latency.

\color{black}
\begin{figure}[h!]
    \centering
    \footnotesize
    \begin{subfigure}[t]{0.49\linewidth} 
        \centering
        \includegraphics[width=\linewidth]{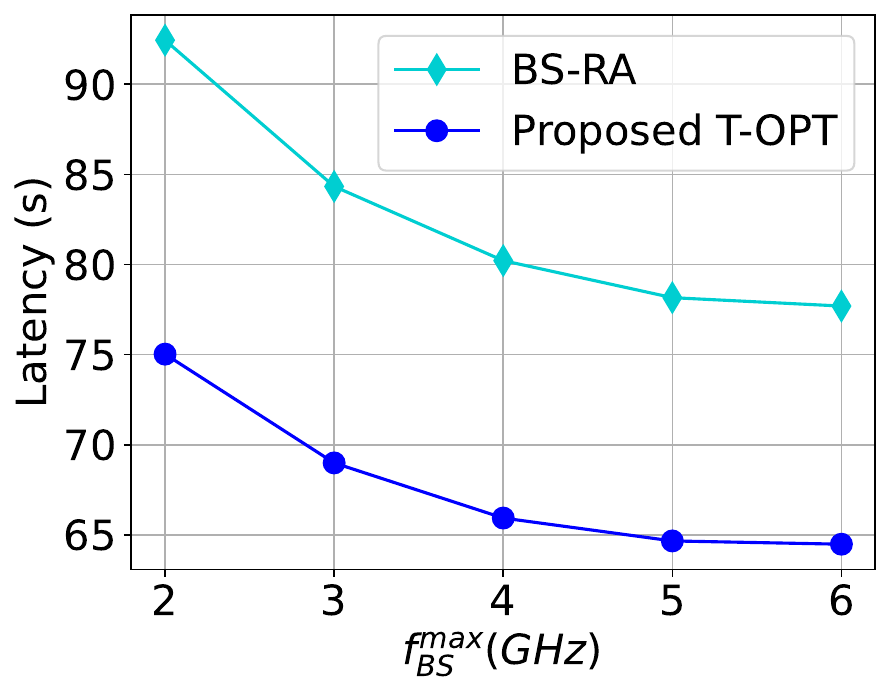}
        \captionsetup{font=footnotesize}
        \caption{\footnotesize System latency against highest value of BS frequency.}
        \label{lat_comp_sub3}
    \end{subfigure}
    \hfill 
    \begin{subfigure}[t]{0.49\linewidth} 
        \centering
        \includegraphics[width=\linewidth]{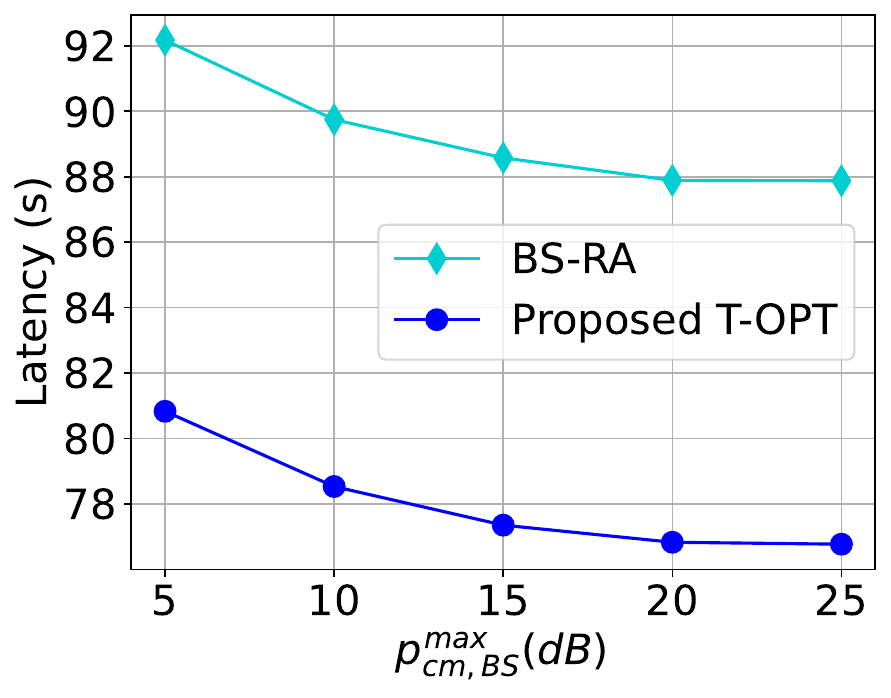}
        \captionsetup{font=footnotesize}
        \caption{\footnotesize System latency versus highest value of BS's transmit power.}
        \label{lat_comp_sub3_bar}
    \end{subfigure}
    \vspace{-3pt}
    \captionsetup{font=footnotesize}
    \caption{\footnotesize Comparison of system latency with scheme BS-RA.}
    \label{lat_comp_sub3_bar_1}
\end{figure}

\begin{figure}[ht!]
    \centering
    \includegraphics[width=0.99\linewidth]{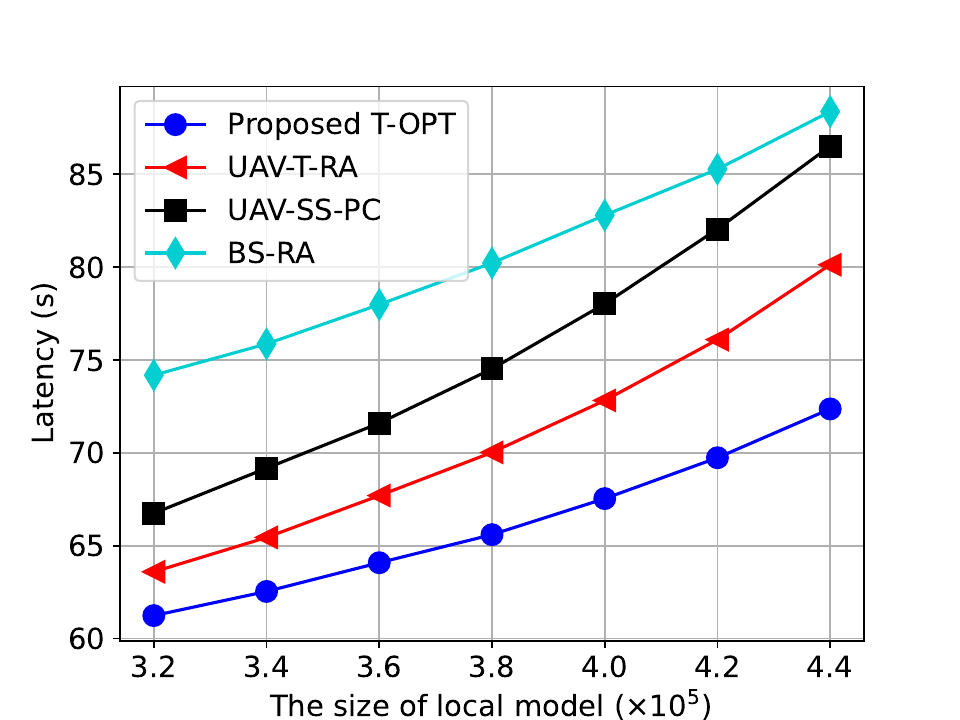}
    \captionsetup{font=footnotesize}
    \caption{\footnotesize Latency comparison of our proposed scheme with other approaches as local model size increases.}
    \label{fig:localmodelsize}
\end{figure}

Fig.~\ref{lat_comp_sub2} shows the relationship between latency (seconds) and the maximum CPU processing rate (in GHz) of UAV, comparing the performance of our proposed algorithm with the UAV-T-RA scheme. Both schemes demonstrate a reduction in latency as UAV frequencies increase, however, our proposed scheme achieves a notable $19.17\%$ decrease in latency compared to scheme UAV-T-RA. reduction in latency compared to UAV-T-RA. Higher frequencies generally lead to faster data transmission rates, contributing to reduced latency by enabling quicker data packet exchanges. The improved performance of our proposed scheme stems from its capacity to adapt dynamically to network conditions, optimizing the parameters for both the UAVs and the BS, which leads to better resource allocation and overall network performance.

Fig.~\ref{lat_comp_sub2_bar} shows the latency (second) versus the maximum communication transmit power of UAV, comparing our proposed scheme with scheme UAV-T-RA. Our proposed scheme achieves $15.94\%$ lower latency compared to scheme UAV-T-RA, although both schemes exhibit reduced latency as the maximum transmit power increases. 


Similarly, Fig.~\ref{lat_comp_sub3} compares the latency of our proposed algorithm T-OPT with BS-RA scheme, presenting latency (second) versus the highest value of BS CPU processing rate (GHz). As the highest value of BS processing rate increases, our proposed scheme outperforms scheme BS-RA, achieving $20.49\%$ reduced latency. Fig.~\ref{lat_comp_sub3_bar} presents latency against highest value of BS communication transmit power, contrasting our proposed algorithm and scheme BS-RA. Our proposed T-OPT method outperforms scheme BS-RA, achieving $14.46\%$ reduced latency.

Fig.~\ref{fig:localmodelsize} presents a comparison of system latency across the schemes UAV-T-RA, UAV-SS-PC, BS-RA, and our proposed T-OPT, as a function of local model size. While latency increases for all schemes with enlarging local model size, the rate of increase is notably lower for our proposed T-OPT scheme. More specifically, T-OPT achieves $10.73\%$, $19.55\%$, and $22.12\%$ lower latency than scheme UAV-T-RA, UAV-SS-PC, and BS-RA, respectively. The figure effectively demonstrates the superior performance of our approach in maintaining lower latency levels under increased computational load.
\begin{figure}[ht!]
    \centering
    \includegraphics[width=0.99\linewidth]{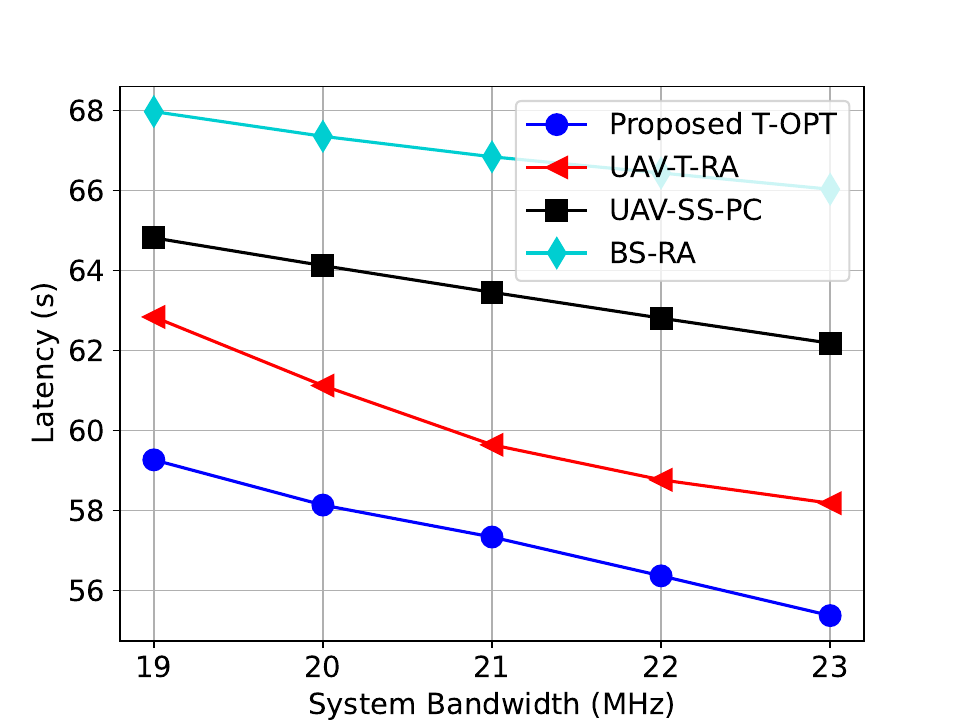}
    \captionsetup{font=footnotesize}
    \caption{\footnotesize Latency comparison of our proposed scheme with other approaches as system bandwidth increases.}
    \label{fig:bandwidth}
\end{figure}
Fig.~\ref{fig:bandwidth} illustrates the impact of increasing system bandwidth on the latency of the aforementioned schemes. From the figure, as the system bandwidth increases, the latency for all schemes decreases, reflecting the enhanced data transmission speeds. However, our proposed T-OPT scheme consistently achieves the lowest latency across all bandwidth scenarios. Specifically, T-OPT shows reductions of $5.07\%$, $12.30\%$, and $19.25\%$ in latency compared to the UAV-T-RA, UAV-SS-PC, and BS-RA schemes, respectively. This figure shows the effectiveness of our joint optimization approach in leveraging increased bandwidth to minimize system latency.

\begin{figure}[ht!]
    \centering
    \includegraphics[width=0.99\linewidth]{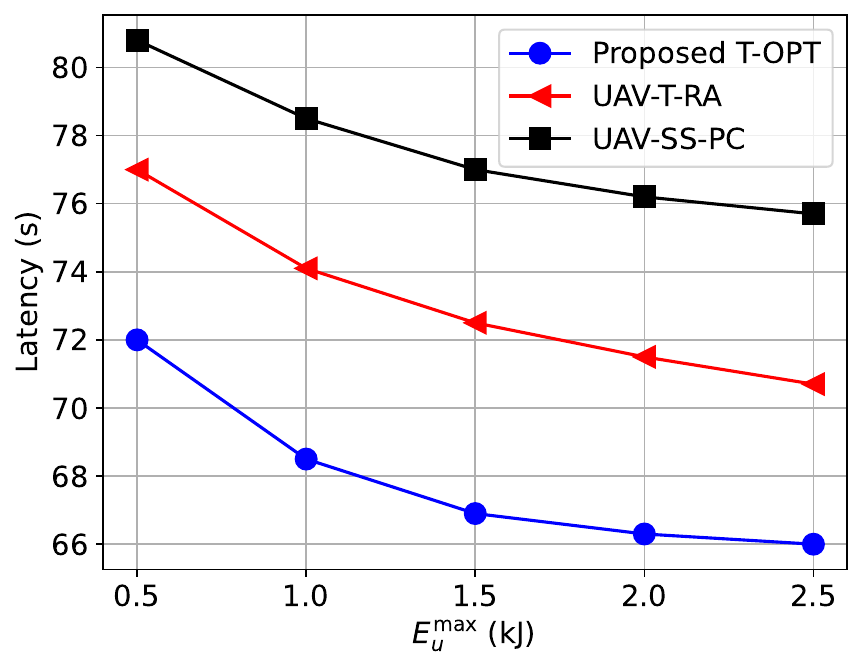}
    \captionsetup{font=footnotesize}
    \caption{\footnotesize Latency comparison of our proposed scheme with other approaches as maximum UAV energy budget increases.}
    \label{fig:energy}
\end{figure}
Fig.~\ref{fig:energy} depicts the impact of increasing the UAV energy budget $E_{u}^{\text{max}}$ on the latency performance of various schemes. As the energy budget increases, all schemes experience reduced latency, owing to the enhanced transmission and computation capabilities of UAVs. Among them, the proposed T-OPT scheme consistently achieves the lowest latency across all energy levels. Compared to UAV-T-RA and UAV-SS-PC, T-OPT achieves latency reductions of up to $6.79\%$ and $12.93\%$, respectively, at the highest energy budget. These results validate the superiority of our joint optimization approach in efficiently utilizing energy resources to minimize system latency.

\begin{figure}[ht!]
    \centering
    \includegraphics[width=0.99\linewidth]{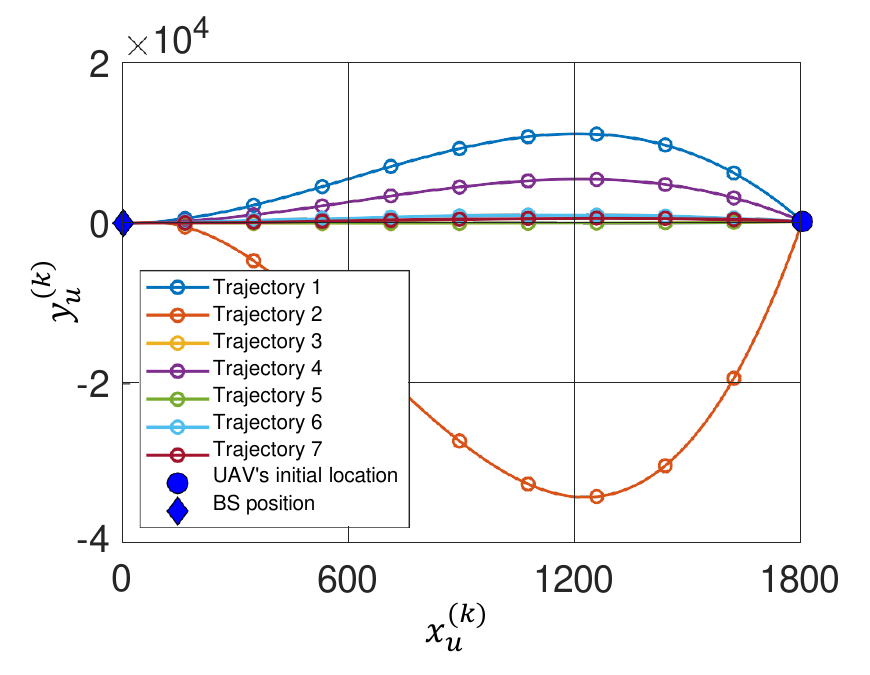}
    \captionsetup{font=footnotesize}
    \caption{\footnotesize UAV flight trajectory optimization through different SCA iterations.}
    \label{fig:traj}
\end{figure}
\color{black}
Fig.~\ref{fig:traj} shows the UAV's horizontal flight trajectories over seven SCA iterations, projected onto the two-dimentional plane defined by $x_{u}^{(k)}$ and $y_{u}^{(k)}$. Starting from the initial location $(1800,0)$ the UAV progressively refines its path toward the BS at $(0,0)$. The trajectories become increasingly straight and efficient, illustrating the SCA algorithm’s convergence toward an latency-minimizing flight path.
 
\section{Conclusion}
This paper has investigated a latency optimization problem of a UAV-enabled FML system, focusing on the joint optimization of UAV sensing scheduling, power control, trajectory, resource allocation, and BS resource allocation. We have provided a comprehensive analysis of the convergence properties of our proposed framework. Our formulated latency minimization problem is extremely challenging to solve because of its non-convex nature. To tackle this, we have proposed an efficient iterative optimization algorithm that combines the BCD and SCA techniques to obtain optimal solutions. Simulation results have shown that our proposed joint optimization method effectively reduces the FML system latency by up to $42.49\%$ compared to baseline methods.

\bibliographystyle{IEEEtran}
\bibliography{Main-Bibliography}
\vspace{-300pt}
\begin{IEEEbiographynophoto}{Shaba Shaon} is currently pursuing her Ph.D. at The University of Alabama in Huntsville, USA. Her research interests include federated learning, quantum computing, and wireless network optimization.
\end{IEEEbiographynophoto}
\vspace{-290pt}
\begin{IEEEbiographynophoto}{Dinh C. Nguyen} is an assistant professor at The University of Alabama in Huntsville, USA. His research interests include quantum computing, federated learning and network security.  He is an Associate Editor of IEEE Transactions on Network Science and Engineering and IEEE Internet of Things Journals.
\end{IEEEbiographynophoto}

\end{document}